\documentclass[10pt,final,twocolumn]{IEEEtran}
\usepackage[utf8]{inputenc}
\usepackage{textcomp}
\DeclareSymbolFont{tildelow}{TS1}{cmr}{m}{n}
\DeclareMathSymbol{\tildelow}{0}{tildelow}{126}
\usepackage{bbm}
\usepackage{xcolor}
\usepackage{color}
\usepackage{graphicx}
\usepackage{epsfig}
\usepackage{subfigure}
\usepackage{amssymb}
\usepackage{amsbsy}
\usepackage{cite}
\usepackage{amsthm}
\usepackage{romannum}
\usepackage{algorithmic}
\usepackage{algorithm}
\usepackage{tikz}

\newtheorem{theorem}{Theorem}
\newtheorem{lemma}{Lemma}

\newtheorem{definition}{Definition}

\newcommand{\beq}{\begin{equation}}
\newcommand{\eeq}{\end{equation}}
\newcommand{\bea}{\begin{array}}
\newcommand{\ena}{\end{array}}
\newcommand{\bds}{\begin {itemize}}
\newcommand{\eds}{\end {itemize}}
\newcommand{\bdf}{\begin{definition}}
\newcommand{\blm}{\begin{lemma}}
\newcommand{\edf}{\end{definition}}
\newcommand{\elm}{\end{lemma}}
\newcommand{\bthm}{\begin{theorem}}
\newcommand{\ethm}{\end{theorem}}
\newcommand{\bprp}{\begin{prop}}
\newcommand{\eprp}{\end{prop}}
\newcommand{\bcl}{\begin{claim}}
\newcommand{\ecl}{\end{claim}}
\newcommand{\bcr}{\begin{coro}}
\newcommand{\ecr}{\end{coro}}
\newcommand{\bquest}{\begin{question}}
\newcommand{\equest}{\end{question}}

\newcommand{\larrow}{{\larrow}}

\def\urltilda{\kern -.15em\lower .7ex\hbox{\~{}}\kern .04em}

\usepackage{float}
\usepackage[T1]{fontenc}
\usepackage{url}
\usepackage{ifthen}
\usepackage{cite}
\usepackage[cmex10]{amsmath} 
\usepackage{graphicx}

\renewcommand{\baselinestretch}{0.955}

\begin{document}

\title{\LARGE Distributed Learning in Markovian Restless Bandits over \\Interference Graphs for Stable Spectrum Sharing}

\author{Liad Lea Didi, Kobi Cohen (\emph{Senior Member, IEEE})
    \thanks{Liad Lea Didi and Kobi Cohen are with the School of Electrical and Computer Engineering, Ben-Gurion University of the Negev, Beer Sheva 8410501 Israel. Email: liadeli@post.bgu.ac.il, yakovsec@bgu.ac.il}
    \thanks{This work was supported by the US-Israel Binational Science Foundation (grant No. 2024611)}
    \thanks{This work has been submitted to the IEEE for possible publication.
Copyright may be transferred without notice, after which this version may
no longer be accessible.}
	\vspace{-0.75cm}
}
\maketitle
\pagenumbering{arabic}

\begin{abstract}
\label{sec:abstract}
We study distributed learning for spectrum access and sharing among multiple cognitive communication entities, such as cells, subnetworks, or cognitive radio users (collectively referred to as cells), in communication-constrained wireless networks modeled by interference graphs. Our goal is to achieve a globally stable and interference-aware channel allocation. Stability is defined through a generalized Gale–Shapley multi-to-one matching, a well-established solution concept in wireless resource allocation. We consider wireless networks where $L$ cells share $S$ orthogonal channels and cannot simultaneously use the same channel as their neighbors. Each channel evolves as an unknown restless Markov process with cell-dependent rewards, making this the first work to establish global Gale–Shapley stability for channel allocation in a stochastic, temporally varying restless environment.

To address this challenge, we develop SMILE (Stable Multi-matching with Interference-aware LEarning), a communication-efficient distributed learning algorithm that integrates restless bandit learning with graph-constrained coordination. SMILE enables cells to distributedly balance exploration of unknown channels with exploitation of learned information. We prove that SMILE converges to the optimal stable allocation and achieves logarithmic regret relative to a genie with full knowledge of expected utilities. Simulations validate the theoretical guarantees and demonstrate SMILE’s robustness, scalability, and efficiency across diverse spectrum-sharing scenarios.
\vspace{0.1cm}
\end{abstract}

\begin{IEEEkeywords}
Spectrum access and sharing, distributed optimization and learning, communication-constrained wireless networks, restless multi-armed bandit (RMAB), Markovian fading channels. \vspace*{-0.2cm}
\end{IEEEkeywords}

\section{Introduction}
\label{sec:introduction}
The major growth in wireless services has intensified the demand for efficient spectrum utilization, pushing modern communication systems toward dynamic and distributed spectrum management. In spatial wireless networks, where spectrum availability fluctuates due to interference caused by transmissions from other cells, achieving high spectral efficiency while minimizing interference remains a central challenge. A natural way to capture the spatial structure of interference is through an interference graph, where nodes represent cells and edges connect pairs that cannot transmit concurrently over the same frequency band. Unlike traditional frequency reuse patterns that rely on static coordination, modern cognitive radio networks operate under stochastic and time-varying environments, requiring autonomous agents to learn and adapt their access decisions over time. This motivates the design of distributed learning algorithms capable of identifying stable and efficient sharing configurations under uncertainty.

In this work, we study a multi-cell spectrum access problem over $S$ orthogonal channels, where $L$ cells compete for transmission opportunities. Each channel is modeled as a Finite-State Markov Channel (FSMC) model, which is independent across channels but not necessarily identically distributed. The FSMC provides a tractable framework commonly employed to describe the temporal dynamics of wireless channels \cite{wang1995finite,sadeghi2008finite}, capturing phenomena such as primary user activity in hierarchical cognitive radio networks and interference from other users in open sharing environments like the ISM bands \cite{zhao2007survey,slamnik2020sharing}. During each time slot, the rate experienced by a cell on a given channel depends on the current FSMC state, while the transition probabilities of the Markov process remain unknown. Each cell can access only one channel per slot and observes the corresponding instantaneous state. To account for spatially localized interference, the network is modeled as an interference graph \cite{srikant2013communication}, in which nodes correspond to cells and edges connect pairs of cells that cannot transmit on the same channel simultaneously. When two neighboring cells select the same channel, a collision occurs, causing their rates to drop to zero. In contrast, non-neighboring cells can reuse the same channel without generating interference.

We evaluate system performance using the stable matching utility (see Section \ref{sec:problem}), a measure known to achieve strong efficiency in multichannel wireless systems \cite{leshem2012multichannel}. \cite{ami2024stable} recently extended this concept to interference graph models, where stability reflects both spatial reuse and local interference constraints. Within this framework, we define regret as the cumulative difference between the achieved and the optimal stable allocation under full knowledge of expected utilities. Our goal is to develop a distributed learning algorithm for spectrum access and sharing that operates under unknown channel dynamics while ensuring sublinear regret growth over time.

While globally stable allocations can be efficiently computed when channel statistics are known, both for fully interfering networks (i.e., a complete interference graph) \cite{leshem2012multichannel} and for spatial interference settings (i.e., general interference graph) \cite{ami2024stable}, a key challenge arises when these statistics are unknown and evolve dynamically. Previous works on multi-player multi-armed bandits (MABs) have addressed distributed learning in fully interfering settings, under rested Markovian dynamics where unobserved channels remain static \cite{kalathil2014decentralized, nayyar2016regret}. These algorithms can achieve near-logarithmic regret of $O (\log t)$ but frequently rely on significant communication between cells to implement auction-based strategies \cite{bertsekas1988auction}, limiting their scalability. Later extensions reduced the communication burden but did not provide formal regret bounds \cite{avner2016multi}. More recently, distributed learning for multi-cell spectrum access under restless Markovian channels was explored under full interference \cite{gafni2022distributed}, but extensions to general spatial interference settings, modeled by arbitrary interference graphs, have not been addressed, which is the focus of this work

Several fully distributed algorithms without communication have been proposed in a fully interfering setting, achieving near $O (\log t)$ regret, but only for i.i.d. channels \cite{bistritz2018distributed}. Other research on Markovian restless MAB (RMAB)\cite{Tekin_2012_Online, liu2012learning, 9288946, cohen2014restless, jiang2023online, raman2024global} considers simplified scenarios involving a single player, or cases where channels provide identical statistics to all cells (homogeneous systems), substantially easing the allocation and analysis. In contrast, our setting involves multi-player heterogeneous restless Markovian channels, in which each cell–channel pair evolves according to an independent Markov process with unknown transition dynamics in a spatial interfering setting. This results in a fundamentally different problem class, requiring multi-player distributed learning over heterogeneous restless dynamics inherent to spatial wireless networks.

\subsection{Main Results}


Our main contributions are summarized as follows:\vspace{0.1cm}

\noindent
\textbf{1) A new general model for distributed spectrum access over interference graphs with restless Markovian channels:} We consider a new practical and general model of the spectrum access and sharing problem in spatial wireless networks, where interference is captured by an arbitrary graph and channel conditions evolve according to heterogeneous restless Markov processes with unknown statistics. Unlike prior works that focus on fully interfering networks, i.i.d. channels, rested dynamics, or homogeneous systems, our model simultaneously captures spatial interference, multi-player competition, and cell-specific restless temporally correlated channel dynamics.

This setting reflects realistic communication network topologies and introduces significant challenges: learning cell-specific expected rates requires sustained exploration of each channel, while interference constraints limit feasible allocations. The problem is formalized as a stable multi-matching problem on interference graphs. This model leads to a formulation as distributed learning in Markovian RMAB, where each arm’s state evolution models the underlying restless channel dynamics, the arm rewards quantify the achievable utility (e.g., rate) on each channel, and the interference graph restricts which arms may be simultaneously selected by the players (cells).
\vspace{0.1cm}

\noindent
\textbf{2) A novel distributed learning algorithm:} We propose a novel distributed learning algorithm, termed Stable Multi-matching with Interference-aware LEarning (SMILE), to address this problem. SMILE carefully balances exploration and exploitation through local sensing and channel contention, requires no global coordination, and can be implemented using lightweight mechanisms such as distributed carrier sensing or local message exchanges between neighboring cells. Unlike prior approaches that either oversample all channels or rely on extensive parameter tuning, SMILE adapts exploration rates online per channel, reducing unnecessary sampling and enabling faster convergence.\vspace{0.1cm}

\noindent
\textbf{3) Rigorous theoretical analysis and performance evaluation:} We provide theoretical analysis showing that SMILE converges to the optimal stable allocation with logarithmic regret relative to an oracle with full knowledge of expected utilities. Compared to previous algorithms \cite{nayyar2016regret, bistritz2018distributed, javanmardi2021decentralized}, SMILE demonstrates improved scaling of regret with respect to both the number of cells and channels, while handling restless Markovian dynamics and arbitrary interference graph constraints. Extensive simulations validate the theoretical results and highlight the efficiency and scalability of SMILE in diverse spectrum sharing scenarios.

\subsection{Other Related Work}

Another significant line of research on multi-cell channel allocation has focused on game-theoretic models, congestion control, related optimization and game-theoretic frameworks (see \cite{han2005fair, leshem2006bargaining, menache2008rate, candogan2009competitive, menache2011network, law2012price, cohen2013game, wu2013fasa, singh2016combined, cohen2016distributedToN, cohen2017distributed, cao2018distributed, bistritz2018approximate, malachi2020queue} and references therein), as well as approaches addressing hidden channel states \cite{yemini2020restless} and classical graph coloring problems (see \cite{checco2017fast} and references therein). In addition to the difference in the learning aspect, the problem studied here differs in several aspects from standard coloring. First, standard graph coloring approaches may be infeasible in practical communication networks, where the number of available channels is limited and not all vertices (cells) can be assigned colors. Second, introducing stability in terms of preferences creates a fundamentally different problem structure, requiring the design of allocation mechanisms that account for both spatial interference and cell-specific utilities. 

Another line of research has investigated spectrum learning under unknown utilities. These approaches include distributed network utility maximization \cite{bistritz2018distributed, bistritz2021game}, stable matching via MAB \cite{avner2016multi}, RMAB \cite{Tekin_2012_Online, liu2012learning, gafni2020learning, gafni2022distributed, gafni2022learning}, and model-free reinforcement learning methods \cite{naparstek2018deep, liu2021dynamic, bokobza2023deep, paul2023multi, cohen2024sinr}. These works demonstrate the potential of learning-based methods for distributed spectrum access, particularly in scenarios where channel statistics are unknown and must be inferred online. However, most of these studies considered simplified or one-to-one assignment settings and do not provide provably stable strategies in the many-to-one allocation setting studied here, where multiple users may compete for the same channel under spatial interference constraints.

\section{Network Model and Problem Statement}
\label{sec:problem}

\subsection{Network Model}
We consider a spectrum access and sharing problem for $L$ cells (i.e., players). At each time, each cell can transmit on one of $S$ channels (i.e., arms) that constitute the spectrum. The set of cells is given by $\mathcal{L}= \{1, 2, \ldots, L\}$, and the set of channels is given by $\mathcal{S}= \{1, 2, \ldots, S\}$. When cell $\ell$ selects a channel $s \in \mathcal{S}$ that is available to it at time $t$, it receives a utility $r_{\ell,s}(t)$, which may represent the achievable transmission rate or a function thereof. We assume that ${r_{\ell,s}(t)}$ evolves as a stochastic process modeled by a discrete-time, irreducible, and aperiodic Markov chain over a finite state space $\mathcal{R}^{\ell,s}$. To capture the time-varying nature of wireless channels, we adopt the restless setting, where channel states evolve regardless of whether they are observed. In particular, we model each channel using an FSMC representation, obtained by quantizing the fading process into a finite number of rate intervals, each corresponding to a state of the Markov chain. The FSMC framework is widely used to characterize temporal channel dynamics \cite{wang1995finite,sadeghi2008finite}, capturing phenomena such as primary-user activity in hierarchical cognitive radio systems and interference in shared-spectrum environments like the ISM bands \cite{zhao2007survey,slamnik2020sharing}. Each channel evolves independently, though the associated Markov chains may follow different transition structures. The Markov chain describing the rate for each cell and channel has transition probability matrix $P^{\ell,s} \triangleq \left( p_{r,r'}^{\ell,s}  : r, r' \in \mathcal{R}^{\ell,s}\right)$ and has a well-defined steady state distribution $\vec{\pi}_{\ell,s}=\{ \pi_{\ell,s}^{r} \}_{r \in \mathcal{R}^{\ell,s}}$. The rate mean is given by $\mu_{\ell,s} = \sum_{r \in \mathcal{R}^{\ell,s}}{r \cdot \pi_{\ell,s}^{r}},$
and it is assumed to be unknown to the cells. These expected rates form an $L \times S$ matrix, denoted by $\textbf{M}$, with entries $[\textbf{M}]_{\ell, s} \triangleq \mu_{\ell, s}$, $\ell=1, \ldots, L$, $s=1, \ldots, S$. 

For each cell $\ell \in \mathcal{L}$, there is a set of neighbors $\mathcal{N}_\ell \subset \mathcal{L}$, consisting of cells with whom $\ell$ cannot transmit simultaneously on the same channel—doing so would result in a zero transmission rate. We denote by  $D_{\ell} = |\mathcal{N}_\ell|$  the size of cell $\ell$'s neighbor set. The channel $s$ is considered free for cell $\ell$ if all cells matched with $s$ are not neighbors of $\ell$. We denote by $x_{\ell,s}(t)$ the actual rate cell $\ell$ experiences by transmitting on channel $s$ at time $t$. If channel $s$ is free for cell $\ell$, then $x_{\ell,s}(t)=r_{\ell,s}(t)$. Otherwise, if a neighbor $\ell' \in \mathcal{N}_{\ell}$ also transmits on $s$ at the same time-slot, a collision occurs and $x_{\ell,s}(t)=0$.

\subsection{Notation}
\label{sec: Notations}
In the following, we define additional expressions and parameters used throughout the paper.
\begin{align*}
    \pi_{\min} \triangleq \min_{\ell \in \mathcal{L}, s \in \mathcal{S}, r \in \mathcal{R}^{\ell, s}} \pi _{\ell, s}^{r}\;\;,\;\;
    \hat{\pi}_{\ell, s}^{r} \triangleq \max \left\{ \pi_{\ell, s}^r, 1-\pi_{\ell, s}^r \right\}.
\end{align*}
\begin{align*}
    \hat{\pi}_{\max} \triangleq \max_{\ell \in \mathcal{L}, s \in \mathcal{S}, r \in \mathcal{R}^{\ell, s}} \left\{ \pi_{\ell, s}^r, 1-\pi_{\ell, s}^r \right\}.
\end{align*}
Also,
\begin{align*}
    r_{\max} \triangleq \max_{\ell \in \mathcal{L}, s \in \mathcal{S}, r \in \mathcal{R}^{\ell, s}} r\;\;,\;\;
    \overline{R}_{\max} \triangleq \max_{\ell \in \mathcal{L}, s \in \mathcal{S}} \sum_{r \in \mathcal{R}^{\ell, s}} r.
\end{align*}
{\small
\begin{align*}
    Q_{\max }\triangleq \max _{\ell, s}\;(\min _{ r \in \mathcal {R}^{\ell, s}}\,\,\pi _{\ell, s}^{r})^{-1} \sum \limits _{r\in {\mathcal {R}^{\ell, s}}}r\;\;,\;\;
    \mathcal{C}_{\max} \triangleq \max_{\ell \in \mathcal{L}, s \in \mathcal{S}} \left\{ |\mathcal{R}^{\ell, s}| \right\}.
\end{align*}}
We denote by $\lambda_{\ell, s}$ the second largest eigenvalue of $P^{\ell,s}$, and the maximum among them across all cells and channels by $\lambda_{\max} \triangleq \max_{\ell \in \mathcal{L}, s \in \mathcal{S}} \lambda_{\ell, s}$. Let $\overline{\lambda}_{\ell, s}= 1 - \lambda_{\ell, s}$ and $\overline{\lambda}_{\min}= 1 - \lambda_{\max}$. At last,
\begin{equation}
\label{I_kappa}
    \kappa \triangleq \frac{28\mathcal{C}_{\max}^2 \overline{R}_{\max}^2 \hat{\pi}_{\max}^2}{\overline{\lambda}_{\min}}\;\;,\;\;
I \triangleq \frac {7\epsilon ^{2} }{48(\overline{R}_{\max }+2)^{2} \cdot \kappa}.
\end{equation}

\subsection{Stable Multi-Matching Formulation}
\label{sec: Stable Allocation}

We adopt the stable matching utility as our performance metric, a criterion shown to yield strong efficiency in multichannel wireless networks \cite{leshem2012multichannel}. This notion originates from the classic stable matching (or stable marriage) problem (SMP) introduced by Gale and Shapley in 1962 \cite{gale1962college}. In particular, the SMP with a common utility formulation was applied in \cite{leshem2012multichannel} to spectrum access in cognitive networks, focusing on one-to-one cell–channel assignments. In this formulation, preferences follow utility (or mean rate) comparisons: cell $\ell$ prefers $s$ over $s'$ if $\mu_{\ell,s}>\mu_{\ell,s'}$, and similarly for channels. More recently, a generalized Gale-Shapley stability was introduced in \cite{ami2024stable} to extend this framework to many-to-one allocations, enabling channel reuse under interference graphs. Importantly, \cite{ami2024stable} also shows that computing the allocation $P:\mathcal{L}\rightarrow\mathcal{S}$ that maximizes the sum rate $\sum_{\ell=1}^L \mu_{\ell,P(\ell)}$ in this setting is NP-hard. Thus, beyond the empirical efficiency of stable solutions, stability emerges as a natural and tractable criterion when optimal rate maximization is computationally infeasible. The generalized Gale-Shapley stability is defined explicitly below.

\begin{definition}[Generalized Gale-Shapley Stable Allocations \cite{ami2024stable}]
An allocation $P: \mathcal{L}\rightarrow \mathcal{S}$ is stable if the following hold:\vspace{0.1cm}  

\noindent
1. (Assignment Validity) Each cell $\ell \in \mathcal{L}$ is assigned to exactly one channel $s \in \mathcal{S}$, and multiple cells may be assigned to the same channel.\vspace{0.1cm}  

\noindent
2. (Interference Feasibility) No two neighboring cells $\ell,\ell' \in \mathcal{N}$ are assigned to the same channel $s \in \mathcal{S}$.\vspace{0.1cm}  

\noindent
3. (Stability) For every cell $\ell_1 \in \mathcal{L}$ that prefers another channel $s \in \mathcal{S}$ over its current assignment $P(\ell_1)$, there exists a neighbor $\ell_2 \in \mathcal{N}_{\ell_1}$ already assigned to $s$ such that $s$ prefers $\ell_2$ over $\ell_1$, i.e., $\mu_{\ell_1,s} < \mu_{\ell_2,s}$.\vspace{0.1cm}
\end{definition}

\subsection{Objective}
\label{sec:Objective}

For each cell $\ell \in \mathcal{L}$, let $\phi_\ell(t)$ denote a selection rule that chooses a channel $s \in \mathcal{S}$ at time $t$ based on the observed history up to time $t-1$. 
A policy $\phi_\ell$ is the sequence of selection rules $\phi_\ell = (\phi_\ell(t),\, t = 1, 2, \ldots)$ governing the channel choices of cell $\ell$. 
Under a given policy, the expected cumulative rate (or utility) of all cells up to time $t$ is given by:
\begin{align}
    R(t) = \mathbb{E} \left[ \sum_{n=1}^t \sum_{\ell=1}^L x_{\ell, \phi_{\ell}(n)}(n) \right].
\end{align}

Our goal is to design a policy that converges to the generalized Gale--Shapley stable allocation defined in Section~\ref{sec: Stable Allocation}. 
Because the channel statistics are unknown, each learner (cell) must estimate the expected rates during operation. 
Let the value of the stable allocation be $\sum_{\ell=1}^L \mu_{\ell, P^*(\ell)}$. 
To assess performance, we use the notion of \emph{regret}, which quantifies the cumulative loss relative to an oracle with perfect knowledge of all expected rates. 
Formally, the regret of a policy $\phi = (\phi_{\ell},\, \ell = 1, 2, \ldots, L)$ is defined as
\begin{align}
\label{eq: regret}
    \mathcal{R}_{\phi}(t) 
    \triangleq 
    t \cdot \sum_{\ell=1}^L \mu_{\ell, P^*(\ell)}
    - 
    \mathbb{E}\!\left[ \sum_{n=1}^{t} \sum_{\ell=1}^L x_{\ell, \phi_{\ell}(n)}(n) \right].
\end{align}

The objective is to design a policy whose time-averaged regret vanishes asymptotically. To this end, we develop an algorithm that efficiently learns the unknown expected rates and converges to the generalized Gale-Shapley stable allocation, ensuring regret grows sublinearly with time. The slower the growth, the stronger the performance.

\section{The SMILE Algorithm}
\label{sec: The Algorithm}

In this section, we present the Stable Multi-matching with Interference-aware LEarning (SMILE) algorithm, which is designed to solve the problem in a distributed manner through three key phases: exploration, allocation, and exploitation. As the expected rates are unknown, the algorithm must gather observations to estimate them, which we refer to as the exploration phase. Using these estimates, the algorithm then seeks the solution to the stable allocation, defined in Section \ref{sec: Stable Allocation}, during the allocation phase, and finally applies this allocation to actual transmissions in the exploitation phase.

Exploration is necessary for accurate rate estimation and ultimately for reducing regret. To achieve reliable estimates, each cell must sample all channels. However, during exploration, cells deviate from the optimal stable allocation, which temporarily increases regret. Conversely, allocation and exploitation phases minimize regret by operating near-optimally with the available estimates. This creates a fundamental trade-off between exploration and exploitation that the algorithm must carefully balance.

To decide which phase to execute at a given time, we use a sampling-based condition. Specifically, let $T_{\ell,s}^{\mathrm{EE}}(t)$ denote the number of estimation samples that cell $\ell$ has collected on channel $s$ during the exploration sub-epochs (detailed in Section~\ref{ssec: exploration phase}) up to time $t$. A cell $\ell$ will enter an exploration epoch on channel $s$ if:
\begin{equation}
\label{eq: exploration condition}
T_{\ell,s}^{\mathrm{EE}}(t) < \tau_{\ell,s}(t),
\end{equation}
where the right-hand side is the exploration function defined in Section~\ref{The Exploration Function}, specific to each cell and channel.
Each cell $\ell$ verifies condition (\ref{eq: exploration condition}) across all channels $s \in \mathcal{S}$. If the condition holds for any channel, the cell enters the exploration phase on that specific channel (details in \ref{ssec: exploration phase}). Otherwise, the cell signals an interrupt message to all cells, indicating readiness to proceed to allocation. Once all cells have signaled, the allocation phase begins to determine the solution to the stable assignment (described in \ref{ssec: allocation phase}). Following that, the exploitation phase takes place, where cells transmit according to the assigned channels. After completing this phase, the entire process repeats. The pseudocode for SMILE is provided in Algorithm~1.
 
\subsection{The Exploration Phase}
\label{ssec: exploration phase}

The purpose of the exploration phase is to collect enough samples to estimate the expected transmission rate that each cell experiences on every channel. These estimates allow cell $\ell$ to identify the $D_{\ell} + 1$ best channels and derive the exploration function, $\tau_{\ell,s}(t)$ that determines when to transition between phases.
Although this phase temporarily increases the regret, since cells transmit on suboptimal channels, it is essential to reduce the long-term regret by enabling accurate rate estimation and eventually identifying the correct stable solution. Rates are estimated by averaging the observed rewards. However, under our problem setting, the environment follows a restless Markovian process, meaning the state may evolve even when the channel is not sampled. Therefore, to ensure that samples are sequential and informative, we divide each exploration phase into two sub-phases: Recovery Epoch (RE) and Estimation Epoch (EE).
In the initial RE, the cell attempts to recover the last observed state to create a continuous sampling process artificially. Once that is achieved, the second sub-phase (EE) begins, in which the cell collects samples for estimation. Let $ N_{\ell,s}^{ER}(t)$ denote the number of exploration phases that cell $\ell$ has conducted on channel $s$ up to time $t$. Let $\xi_{\ell, a}(N_{\ell,s}^{ER}(t))$ denote the last state observed during the $N_{\ell,s}^{ER}(t)$ -th exploration epoch. The first sub-epoch continues until the observed state equals $\xi_{\ell, a}(N_{\ell,s}^{ER}(t))$, resulting in a random-length phase. Then, the second sub-epoch runs for a fixed duration of $4^{N_{\ell,s}^{ER}(t)}$.
The length of the first sub-epoch increases geometrically to reduce channel switching, in a way that avoids regret increase caused by the transient effect, while ensuring the total time spent on exploration is logarithmic (see (\ref{proof ssec total exploration time})). At the end of each exploration phase, the estimated rate is computed as the sum of all observed states across all estimation sub-epochs (EE), denoted by \( S_{\ell,s}^{\mathrm{EE}}(t) \), divided by the total number of estimation samples collected so far during these epochs, $T_{\ell,s}^{\mathrm{EE}}(t)$. As a result, the estimated mean rate for cell $\ell$ on channel $s$ is given by:
$\hat{r}_{\ell,s}(t) = \frac{S_{\ell,s}^{\mathrm{EE}}(t)}{T_{\ell,s}^{\mathrm{EE}}(t)}.$
An illustration of the exploration process for cell $\ell$ on channel $s$ is shown in Fig.~\ref{fig: SMILE phases}.

\begin{figure}[!t]
\centering
\includegraphics[width=0.9\linewidth]{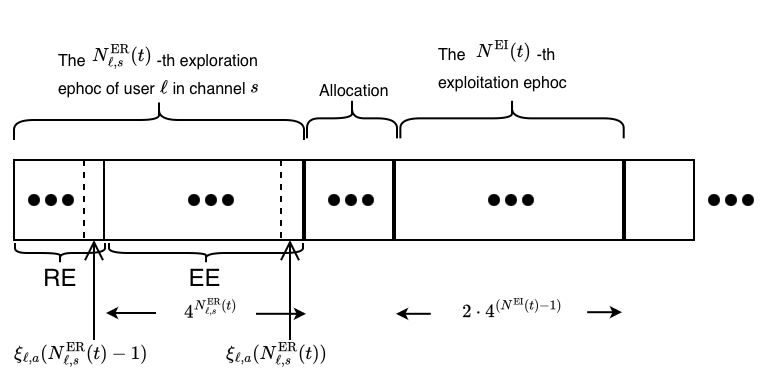}
\caption{An illustration of the SMILE algorithm phases for cell $\ell$.}
\label{fig: SMILE phases}
\end{figure}

\subsubsection{The Exploration Function}
\label{The Exploration Function}

To minimize the regret, the algorithm must both apply the stable allocation and collect sufficient exploration samples to ensure accurate rate estimation. Thus, we aim to define an exploration function that determines when cell $\ell$ should explore channel $s$.

Next, we introduce the \emph{Exploration Coefficient} used in the development of the exploration function. Note that to distinguish (with high probability) between two rate estimates (e.g., rate of cell $\ell$ on channels $s$ and $p$), at least $\frac{4 \kappa}{(\mu_{\ell, s}-\mu_{\ell, p})^2}$ samples are needed for each.
The primary goal of exploration is to allow each cell $\ell$ to correctly identify the $D_{\ell} + 1$ best channels, denoted by $\mathcal{S}_{\ell}$. Therefore, for each channel $s \in \mathcal{S}_{\ell}$, we define a deterministic row exploration coefficient:
\begin{equation}
E_{\ell, s}^{(R)} \triangleq \frac{4 \kappa}{\min \limits_{p \neq s}(\mu_{\ell, s}-\mu_{\ell, p})^2},
\end{equation}
and for channel $s \notin \mathcal{S}_{\ell}$:
\begin{equation}
E_{\ell, s}^{(R)} \triangleq \frac{4 \kappa}{(\mu_{\ell, s}-\min \limits_{p \in \mathcal{S}_{\ell}}\mu_{\ell, p})^2}.
\end{equation}
Note that $\min_{p \in \mathcal{S}_{\ell}} \mu_{\ell, p}$ corresponds to the $(D_{\ell} + 1)$-th largest mean rate that cell $\ell$ experiences across the channels.
Additionally, to determine whether a cell has a higher estimated rate than any neighbor that is also interested in a given channel, we introduce a deterministic column-exploration coefficient defined as:
\begin{equation}
E_{\ell, s}^{(C)} \triangleq \frac{4 \kappa}{\min \limits_{q \in \mathcal{V}_{\ell, s}}(\mu_{\ell, s}-\mu_{q, s})^2}.
\end{equation}
Here, $\mathcal{V}_{\ell, s}$ denotes the set of neighbors of cell $\ell$ that have attempted to transmit on channel $s$ (see Section~\ref{ssec: allocation phase}). Note that, by the design of the Allocation Phase, the estimated rates that cells from this set will experience on channel $s$ are known.
Combining both gives us the overall deterministic exploration coefficient:
\begin{equation}
E_{\ell, s} \triangleq \max\{E_{\ell, s}^{(R)}, E_{\ell, s}^{(C)}\}.
\end{equation}

Next, we introduce the \emph{Estimated Exploration Coefficient} used in the development of the exploration function. Since the true expectations are unknown, we replace them with current estimates in the implementation, such that for each channel $s \in \mathcal{S}_{\ell}$ we have:
\begin{equation}
\hat{E}_{\ell, s}^{(R)}(t) \triangleq \frac{4 \kappa}{\max\{\Delta_{\text{min}}^2, \min \limits_{p \neq s}(\hat{r}_{\ell,s}(t)-\hat{r}_{\ell,p}(t))^2 -\epsilon \}} ,
\end{equation}
and for channel $s \notin \mathcal{S}_{\ell}$ we have:
\begin{equation}
\hat{E}_{\ell, s}^{(R)}(t) \triangleq \frac{4 \kappa}{\max\{\Delta_{\text{min}}^2, (\hat{r}_{\ell,s}(t)-\min \limits_{p \in \mathcal{S}_{\ell}}\hat{r}_{\ell,p}(t))^2 -\epsilon \}}.
\end{equation}
For the column coefficient, we define:
\begin{equation}
\hat{E}_{\ell, s}^{(C)}(t) \triangleq \frac{4 \kappa}{\max\{ \Delta_{\text{min}}^2, \min \limits_{q \in \mathcal{V}_{\ell, s}} (\hat{r}_{\ell,s}(t)-\hat{r}_{q,s }(t))^2 -\epsilon \} }.
\end{equation}
Finally, the estimated exploration coefficient is given by:
\begin{equation}
\label{eq E estimation}
\hat{E}_{\ell, s}(t) \triangleq \max\{ \hat{E}_{\ell, s}^{(R)}(t), \hat{E}_{\ell, s}^{(C)}(t)\} ,
\end{equation}
where
\begin{equation}
\label{delta}
\Delta_{\text{min}} \triangleq \min\{ \min_{\ell \in \mathcal{L}} \Delta^{(R)}_{\ell} , \min_{s \in \mathcal{S}} \Delta^{(C)}_{s}\},
\end{equation}
\begin{equation}
\Delta^{(R)}_{\ell} \triangleq \min_{s \neq p} |\mu_{\ell, s}-\mu_{\ell, p}| , \; \Delta^{(C)}_{s} \triangleq \min_{\ell \in \mathcal{L} , q \in \mathcal{V}_{\ell,s}} |\mu_{\ell, s}-\mu_{q, s}|.
\end{equation}
We note that $\Delta_{\text{min}}$ and $\epsilon$ are only needed for purposes of analysis.

Next, we introduce the \emph{Exploration Function}. 
To guarantee the desired convergence rate, we require at least $2/I\cdot\log (t)$ samples for each cell and channel (see (\ref{proof eq 6})). As a result, the exploration function is given by:
\begin{equation}
\label{eq: exploration function}
\tau_{\ell,s}(t) \triangleq \max \{\hat{E}_{\ell, s}(t), \frac{2}{I} \} \cdot \log (t).
\end{equation}
As long as condition (\ref{eq: exploration condition}) holds, cell $\ell$ will continue to an exploration phase on channel $s$.

\subsection{The Access Phase (Allocation)}
\label{ssec: allocation phase}
The allocation phase in SMILE has two main objectives. The first is to compute the distributed generalized Gale–Shapley multi-to-one matching under the noisy learned rates. The second is to share limited information required for having the exploration coefficient. Note that since the true rates are unknown, the algorithm relies on the estimated rates obtained during the exploration phases to improve the learning process and converge to the global solution.
The allocation phase consists of at most $L \cdot S$ iterations (until all cells are assigned to specific channels). Each iteration is divided into two sub-phases, $S_1$ and $S_2$.
In each iteration, during sub-phase $S_1$, the unassigned cell with the highest average rate on a given channel attempts to transmit over that channel. If one (or more) of its neighbors is already assigned and transmitting on that channel, it implies that the neighbor’s rate is higher, and hence the neighbor is already assigned to that channel. In the subsequent sub-phase $S_2$, the previously unassigned cell that attempted to transmit in $S_1$ transmits again. Otherwise, the channel is free from any neighbor’s transmission, and the cell becomes assigned to it.
Formally, let $\hat{R}(t)$ denote the matrix of estimated rates at time $t$ (that is, $[\hat{R}(t)]_{\ell, s}$ represents the estimated mean rate of cell $\ell$ on channel $s$). Initially, this matrix is set according to the current mean-rate estimates.
Let us also define the set of neighbors of cell $\ell$ that, at some point during the allocation phase, transmitted simultaneously with $\ell$ on channel $s$ (during sub-epoch $S1$) and caused a collision (this set is required for computing the column exploration coefficient) by $\mathcal{V}_{\ell, s}$.
In each iteration, during sub-phase $S_1$, the maximal entry $(\ell, s)$ of $\hat{R}(t)$ is identified, and cell $\ell$ attempts to transmit on channel $s$. If one (or more) of its neighbors is already assigned to channel $s$, a collision occurs — meaning that the neighbor(s) was assigned earlier, and cell $\ell$ remains unassigned. In that case, cell $\ell$ stores the indices of the transmitting neighbor(s) on channel $s$ in $\mathcal{V}_{\ell, s}$ and their corresponding estimated rates.
In sub-phase $S_2$, cell $\ell$ transmits again on channel $s$, and any neighbor(s) already assigned to this channel, $q$, store index $\ell$ in $\mathcal{V}_{q, s}$ and its rate. Otherwise, if no neighbor is currently assigned to $s$, channel $s$ is considered free, and cell $\ell$ becomes assigned to it. Consequently, all related entries in row $\ell$ of $\hat{R}(t)$ are set to zero.
This process repeats until the entire matrix $\hat{R}(t)$ becomes zero, or equivalently, until all cells are assigned to their respective channels.

\subsection{Distributed Implementation and Illustrative Example}

The distributed coordination required by SMILE can be carried out in communication networks through two simple mechanisms:

\noindent
\textbf{1) Opportunistic carrier sensing multiple access (CSMA) over the neighborhood graph.} 
    This mechanism relies on opportunistic CSMA techniques proposed for distributed spectrum access in communication networks \cite{zhao2005opportunistic, cohen2010time, leshem2012multichannel, cohen2019time}. 
    Each cell applies a rate-dependent backoff timer on channel $s$, where higher estimated rates correspond to shorter backoff durations. 
    Consequently, the cell with the highest estimated rate on a given channel transmits first, while its neighbors detect the channel as busy and refrain from transmitting. 
    In this setting, Subphase~$S_1$ corresponds to each unassigned cell running opportunistic CSMA on its best remaining channel, and Subphase~$S_2$ enables assigned neighbors to record the identities of attempting cells for constructing $\mathcal{V}_{\ell,s}$.

\noindent
\textbf{2) Local message exchange between neighboring cells.}  
    In this implementation, each cell broadcasts a short message to its neighbors indicating its attempt to transmit on channel $s$. 
    Neighboring cells compare their estimated rates and suppress their own transmission attempts if their rate is lower, effectively yielding the channel to the higher-rate cell. 
    Under this approach, Subphase~$S_2$ is unnecessary, as the identities and rates of competing cells are already exchanged explicitly.\vspace{0.2cm}

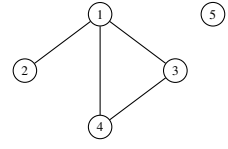
\begin{figure}[t]
    \centering
    \begin{tikzpicture}[
        scale=0.5, 
        every node/.style={circle, draw, minimum size=0.6cm, transform shape}
    ]
        \node (1) at (0,0) {1};
        \node (2) at (-2,-1.5) {2};
        \node (3) at (2,-1.5) {3};
        \node (4) at (0,-3) {4};
        \node (5) at (3,0) {5};
        
        \draw (1) -- (2);
        \draw (1) -- (3);
        \draw (1) -- (4);
        \draw (3) -- (4);
    \end{tikzpicture}
    \caption{Interference graph topology of 5 cells and 3 channels. 
    Edges indicate interference between cells.}
    \label{fig:topology}
\end{figure}

To illustrate the allocation phase, consider an example with five cells and three channels, under the neighborhood topology shown in Fig.~\ref{fig:topology}. The estimated mean rates at the beginning of the allocation phase are depicted in Fig.~\ref{fig:Alloc_exmp_rates}.
Fig.~\ref{fig:Alloc_exmp_iters} presents all the iterations of the current allocation phase, including its two subphases, S1 and S2. In the first iteration, the unassigned cell with the highest estimated rate is cell 4 on channel 3, which is free from transmissions by its neighbors (cells 1 and 3). Therefore, cell 4 is assigned to transmit on this channel and sets its rate to zero on all other channels in its corresponding row in Fig.~\ref{fig:Alloc_exmp_rates}. Its neighbors update its index and estimated rate.
In the second iteration, cell 3 attempts to transmit on channel 3. However, since this channel is already occupied by its neighbor (cell 4), a collision occurs and cell 3 is not assigned to this channel. During Subphase S2, cell 3 transmits on channel 3, while cell 4 records its index and rate.
In the third iteration, cell 2 is assigned to channel 3, as none of its neighbors (cell 1) transmits on it. In the fourth iteration, cell 5 is assigned to channel 2. In the fifth iteration, cell 1 is matched to channel 1. In the sixth iteration, cell 3 attempts to transmit on channel 1, but since it is already occupied by its neighbor (cell 1), it transmits in S2 instead.
Finally, in the last iteration, cell 3 is successfully assigned to channel 2, completing the allocation phase after seven iterations and nine time indices.

\begin{figure}[htbp]
    \centering
    \includegraphics[width=0.7\columnwidth]{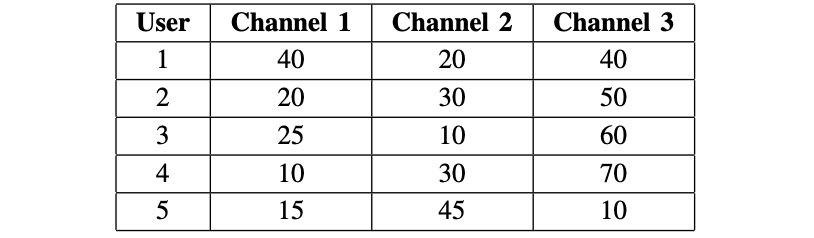}
    \caption{Estimated expected rate matrix [Mbps].}
    \label{fig:Alloc_exmp_rates}
\end{figure}

\begin{figure}[htbp]
    \centering
    \includegraphics[width=0.7\columnwidth]{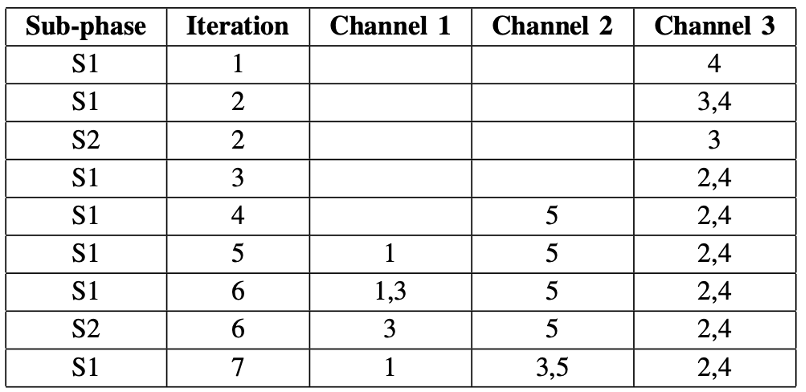}
    \caption{Allocation phase iterations.}
    \label{fig:Alloc_exmp_iters}
\end{figure}

\begin{algorithm}[h]
\scriptsize
\caption{SMILE Algorithm for cell $\ell$}
\begin{algorithmic}[1]

\STATE Set $\epsilon>0$ ; $t=0$, $N_{\ell,s}^{ER}=1$, $N^{EI}=0$ ; $T_{\ell,s}^{\mathrm{EE}}=0$; $S_{\ell,s}^{\mathrm{EE}}=0, \, \forall s=1\ldots S$
\STATE Initialization

\FOR{$s=1:S$}
  \STATE access channel $s$ ; denote observed state and rate as $x$ and $r_x$ , respectively, and set $\xi_{\ell, s}(N_{\ell,s}^{ER})=x$
  \STATE $t:=t+1$ ; $ T_{\ell,s}^{\mathrm{EE}} := T_{\ell,s}^{\mathrm{EE}}+1$ ; $ N_{\ell,s}^{ER} := N_{\ell,s}^{ER}+1$ ; $ S_{\ell,s}^{\mathrm{EE}} := S_{\ell,s}^{\mathrm{EE}}+r_x$
  \STATE $\hat{r}_{\ell,s} = \frac{S_{\ell,s}^{\mathrm{EE}}}{T_{\ell,s}^{\mathrm{EE}}}$
\ENDFOR

\WHILE{true}
  \FOR{$s=1:S$}
    \STATE estimate $E_{\ell, s}$ according to (\ref{eq E estimation})
  \ENDFOR

  \WHILE{condition (\ref{eq: exploration condition}) holds for some channel $s$}
    \STATE \textbf{Exploration Phase Algorithm}
    \STATE access channel $s$ ; denote observed state as $x$
    \WHILE{$x \neq \xi_{\ell, s}(N_{\ell,s}^{ER}-1)$}
      \STATE \textbf{RE Epoch:}
      \STATE $t:=t+1$
      \STATE access channel $s$ ; denote observed state and rate as $x$ and $r_x$
    \ENDWHILE

    \STATE $t:=t+1$ ; $ T_{\ell,s}^{\mathrm{EE}} := T_{\ell,s}^{\mathrm{EE}}+1$ ; $ S_{\ell,s}^{\mathrm{EE}} := S_{\ell,s}^{\mathrm{EE}}+r_x$

    \FOR{$n=1: 4^{N_{\ell,s}^{ER}-1}$}
      \STATE \textbf{EE Epoch:}
      \STATE access channel $s$ ; denote observed state and rate as $x$ and $r_x$
      \STATE $t:=t+1$ ; $ T_{\ell,s}^{\mathrm{EE}} := T_{\ell,s}^{\mathrm{EE}}+1$ ; $ S_{\ell,s}^{\mathrm{EE}} := S_{\ell,s}^{\mathrm{EE}}+r_x$
    \ENDFOR

    \STATE $N_{\ell,s}^{ER} := N_{\ell,s}^{ER}+1$
    \STATE $\hat{r}_{\ell,s} = \frac{S_{\ell,s}^{\mathrm{EE}}}{T_{\ell,s}^{\mathrm{EE}}}$
    \STATE $\xi_{\ell, s}(N_{\ell,s}^{ER})=x$
  \ENDWHILE

  \STATE Send an interrupt signal
  \IF{Interruption signal is not observed from all cells}
    \STATE goto step 10
  \ENDIF

  \STATE \textbf{Allocation Phase Algorithm}
  \STATE Start an allocation phase according to Sec.~\ref{ssec: allocation phase}. Denote assigned channel as $s_A$

  \STATE \textbf{Exploitation Phase Algorithm}
  \FOR{$n=1: 2\cdot 4^{(N^{EI}-1)}$}
    \IF{Interruption signal is observed}
      \STATE goto step 10
    \ENDIF
    \STATE Access channel $s_A$. Denote observed state and rate as $x$ and $r_x$
    \STATE $t:=t+1$
  \ENDFOR
  \STATE $N^{EI} := N^{EI}+1$
\ENDWHILE

\end{algorithmic}
\end{algorithm}

\subsection{The Stable Spectrum Sharing Phase (Exploitation)}
\label{ssec: exploitation phase}
The goal of this phase is to implement the stable allocation (identified at the end of the allocation phase) based on the rate estimates obtained during the estimation epochs within the exploration phases. Each cell transmits throughout the entire phase on the channel to which it was matched during the allocation phase. As the estimated rates converge to the true mean rates, the resulting stable allocation asymptotically approaches the optimal solution (as if the mean rates were fully known), thereby reducing the overall regret.
To ensure that the total durations of the exploration and allocation phases remain logarithmic in time, we execute each exploitation phase for a period of $2\cdot 4^{(N^{EI}(t)-1)}$, where $N^{EI}(t)$ denotes the number of exploitation phases completed up to time $t$.

\section{Theoretical Regret Analysis}
\label{sec:performance}

To evaluate the learning efficiency of SMILE, we analyze its regret relative to the genie-optimal stable allocation. In RMAB problems, sublinear regret is essential for ensuring asymptotic optimality. The following theorem establishes the regret bound achieved by SMILE.

\begin{theorem}
\label{theorem 1}
    Assuming that the proposed SMILE algorithm is implemented and that the assumptions on the system model in Section~\ref{sec:problem} hold. Then, the regret at time $t$ is upper bounded by:
    {\small
    \begin{align} 
    \label{eq: regret bound}
    \mathcal {R}(t)\leq& Q_{\max } \cdot \left({\sum \limits _{\ell=1}^{L} \sum \limits _{s=1}^{S} (\lfloor \log _{4}(3\mathcal{E}_{\ell, s}\log (t)+1) \rfloor +1) }\right) 
    \notag \\ 
    &+\, \sum \limits _{\ell=1}^{L} \sum \limits _{s=1}^{S} \bigg [\bigg (4\mathcal{E}_{\ell, s} \cdot \log (t) + 1 
    \notag \\ 
    &  +\, M_{\max }^{\ell, s}\big (\lfloor \log _{4}(3\mathcal{E}_{\ell, s}\log (t)+1) \rfloor +1 \big) \bigg) 
    \notag \\ 
    &  \cdot \bigg (\mu _{\ell, P(\ell)} + \! \! \! \! \sum_{q \in P^{-1}(s) \cap \mathcal{N}_\ell } \! \! \! \! \left[ \mu _{q,s} \right] \!\! - \mu _{\ell, s} \bigg) \bigg] 
    \notag \\ 
    &+\, 2LS \cdot Q_{\max }\! \cdot \! \left({\!\sum \limits _{\ell=1}^{L} \!\sum \limits _{s=1}^{S} \!(\lfloor \log _{4}(3\mathcal{E}_{\ell,s}\log (t)\!+\!1) \rfloor \!+\!1) \!}\right) 
    \notag 
    \\ \hspace{-0cm}
    &
    +\, \bigg [\!\bigg (\!2 \!\cdot \! L S \!\bigg) 
    \notag 
    \cdot \! \left({\!\sum \limits _{\ell=1}^{L} \!\sum \limits _{s=1}^{S} (\lfloor \log _{4}(3\mathcal{E}_{\ell, s} \log (t)+1) \rfloor +1) }\right) \bigg] 
    \notag \\ 
    &\cdot \left[{\sum \limits _{q=1}^{L} \mu _{q, P(q)} }\right] 
    \notag 
    +\, \Bigg(L \cdot Q_{\max } \notag \\ 
    &+\, (L S \cdot (\max_{\ell \in \mathcal{L}} D_{\ell})+LS) \frac {4\mathcal{C}_{\max }}{\pi _{\min }}\left({\sum \limits _{q=1}^{L} \mu _{q,P(q)}}\right) \Bigg) 
    \notag \\  
    &\cdot \left({\left\lceil \log _{4}\left({\frac {3}{2}t+1}\right) \right\rceil }\right) +O(1),
    \end{align}}
\end{theorem}
where $M^{\ell, s}_{r,r'}$ denotes the mean hitting time of state $r'$ starting from state $r$ for channel $s$ used by cell $\ell$, $M^{\ell, s}_{\max }\triangleq \max \limits _{r,r' \in \mathcal {R}^{\ell, s}, r\neq r'}M^{\ell, s}_{r,r'}$ and $\mathcal{E}_{\ell, s}$ is given by:

\begin{align} 
\label{eq: {E}_l,s}
\mathcal{E}_{\ell, s}\triangleq \begin{cases} \max \{2/I\;,\;E_{\ell,s}^{(\max)} \} \;, & \text{if} \; s\in \mathcal {A}_{\ell} \\ \max \{2/I\;,\;4\kappa/\Delta _{\min }^{2}\} \;, & \text{if} \; s {\not \in }\mathcal {A}_{\ell} \end{cases}. 
\end{align}
\noindent
The set $\mathcal{A}_{\ell}$ consists of all indices $s \in \mathcal{S}$ of cell $\ell$ that for $s \in \mathcal{S}_{\ell}$ satisfy
{\small
\begin{equation*}
\min \{\min _{p \neq s} \{ ( \mu _{\ell, s}-\mu _{\ell, p} )^{2} \}, \min _{q \in \mathcal{V}_{\ell, s}} \{ ( \mu _{\ell, s}-\mu _{q, s} )^{2}\} \} -2\epsilon > \Delta _{\min }^{2},
\end{equation*}
}
and for $s \notin \mathcal{S}_{\ell}$ satisfy
\begin{equation*} \min _{p \neq s} \{ ( \mu _{\ell, s}-\mu _{\ell, p} )^{2} \} -2\epsilon > \Delta _{\min }^{2},\end{equation*}
where $E^{(\max)}_{\ell, s}$ is defined as:
{\small
\begin{align}
\label{eq: E_l,s max}
&\hspace {-1.3pc}E_{\ell, s}^{(\max)} \notag 
\triangleq
\frac {4\kappa}{\min \big \{\min \limits _{p \neq s} \{ (\mu _{\ell, s}\!-\!\mu _{\ell, p}) \}^{2}, \min \limits_{q \in \mathcal{V}_{\ell, s}} \{ ( \mu _{\ell, s}-\mu _{q, s} )^{2}\} \big \} \!-\!2\epsilon }
\end{align}
}
The proof is given in the Appendix.\vspace{0.1cm}

As seen in the theorem, SMILE achieves logarithmic regret with time, ensuring that its performance converges asymptotically to that of a genie with complete knowledge of the expected rates. This regret order is the best attainable in problems of this class, providing strong theoretical guarantees for the efficiency of the algorithm.

\section{Simulation Results}

In this section, we present extensive simulations to evaluate the efficiency of SMILE. We begin by demonstrating its convergence to the optimal centralized solution of the stable allocation. We then evaluate SMILE’s learning performance relative to state-of-the-art RMAB-based methods.

First, we consider a case with $L=3$ cells and $S=5$ channels, where cells 1 and 2 are neighbors and cell 3 has no neighbors. The wireless channels follow Rayleigh fading and are modeled using an FSMC with $N=6$ quantized states. The resulting transition probability matrix $P$ and the expected rate matrix $\mathbf{M}$ are:
{\small
\[
P =
\begin{pmatrix}
3/6 & 2/6 & 1/6 & 0 & 0 & 0 \\
2/8 & 3/8 & 2/8 & 1/8 & 0 & 0 \\
1/9 & 2/9 & 3/9 & 2/9 & 1/9 & 0 \\
0 & 1/9 & 2/9 & 3/9 & 2/9 & 1/9 \\
0 & 0 & 1/8 & 2/8 & 3/8 & 2/8 \\
0 & 0 & 0 & 1/6 & 2/6 & 3/6
\end{pmatrix},
\]
}

{\small
\[
\textbf{M} =
\begin{pmatrix}
45 & 10 & 35 & 25 & 80 \\
30 & 45 & 20 & 75 & 90 \\
55 & 5 & 70 & 15 & 45
\end{pmatrix}.
\]
}

In Fig.~\ref{sim:sim1}, we show the average sum rate achieved by SMILE. We compare its performance to the optimal centralized solver (an oracle with full knowledge of the mean rates) to illustrate convergence, and to a random allocation baseline to highlight the benefit of learning. As the figure shows, SMILE rapidly approaches the centralized solver’s performance and substantially outperforms the random allocation.

\begin{figure}[htbp]
	\centering \epsfig{file=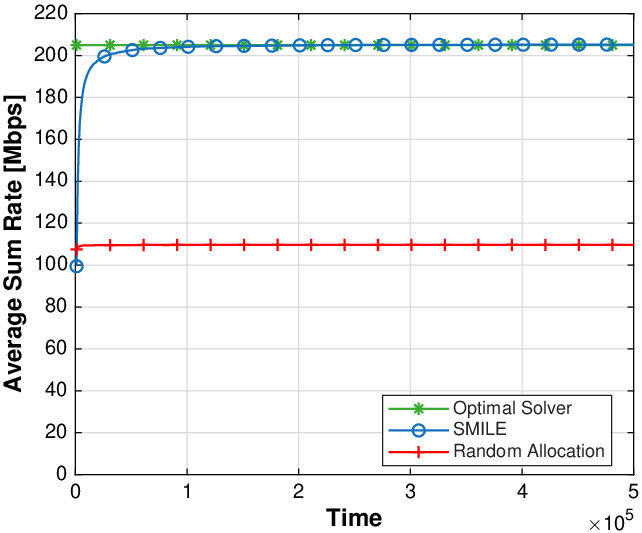,
    width=0.3\textwidth}
  \caption{The average sum rate as a function of time with $L=3$, $S=5$ under SMILE, the optimal stable allocation (benchmark), and random allocation.}
  \label{sim:sim1}
  \end{figure}

Next, we considered a network with a large number of cells to evaluate SMILE’s performance in large-scale settings. As before, we compared SMILE against the centralized optimal solver (oracle), which, having full knowledge of the mean rates, computes the stable allocation. We observed clear convergence of the achieved sum rate both in a system with 50 cells and 50 channels (Fig.~\ref{sim:sim4}) and in an even larger configuration with 100 cells and 100 channels (Fig.~\ref{sim:sim5}). We then examined the case of 100 cells and 50 channels (Fig.~\ref{sim:sim6}), where the interference graph enables channel reuse. As the figures show, SMILE consistently and rapidly approaches the performance of the centralized solver across all these large-scale scenarios.

\begin{figure}[htbp]
	\centering \epsfig{file=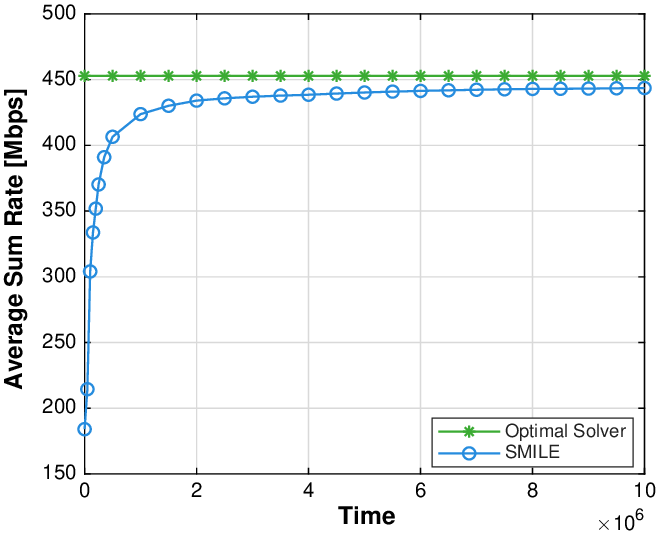,
    width=0.3\textwidth}
  \caption{The average sum rate as a function of time with $L=50$, $S=50$ under SMILE and the optimal stable allocation (benchmark).}
  \label{sim:sim4}
  \end{figure}

\begin{figure}[htbp]
	\centering \epsfig{file=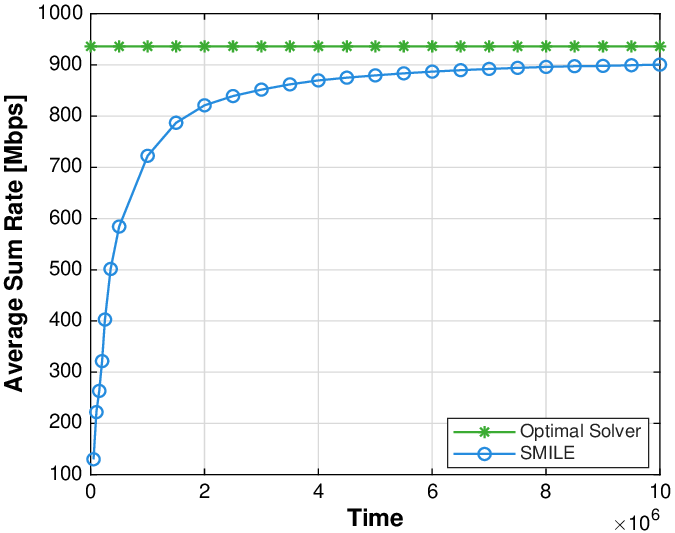,
    width=0.3\textwidth}
  \caption{The average sum rate as a function of time with $L=100$, $S=100$ under SMILE and the optimal stable allocation (benchmark).}
  \label{sim:sim5}
  \end{figure}

\begin{figure}[htbp]
	\centering \epsfig{file=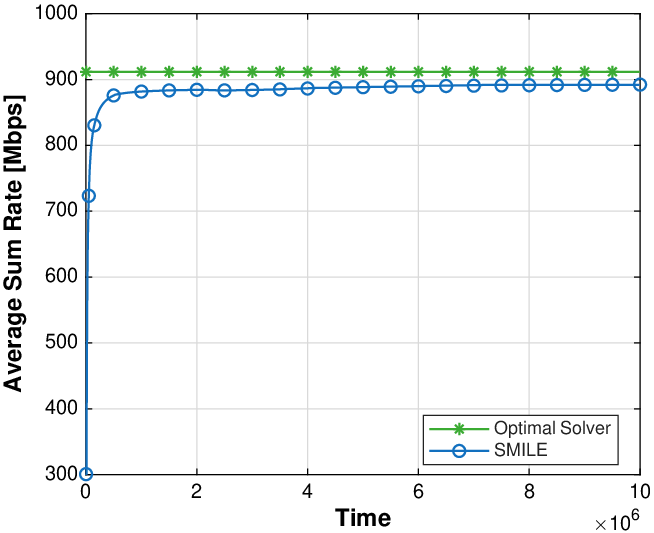,
    width=0.3\textwidth}
  \caption{The average sum rate as a function of time with $L=100$, $S=50$ under SMILE and the optimal stable allocation (benchmark).}
  \label{sim:sim6}
  \end{figure}

We next evaluated the learning efficiency of SMILE against state-of-the-art RMAB-based algorithms. The simulation setting follows a hierarchical spectrum-access model in which primary and secondary users share the spectrum. Primary users intermittently occupy each channel, and secondary users are allowed to transmit only when the channel is idle. Each channel alternates between a “good’’ state, offering a positive expected rate, and a “bad’’ state, yielding zero rate. The temporal behavior of primary-user activity is captured by a Gilbert–Elliott Markov model. In Fig.~\ref{sim:sim2}, we compare SMILE with three well-known RMAB learning algorithms—RCA~\cite{Tekin_2012_Online}, DSEE~\cite{liu2012learning}, and DSSL~\cite{gafni2022distributed}. As shown, SMILE attains substantially lower regret than all competing methods, highlighting its superior learning efficiency in RMAB environments.

\begin{figure}[htbp]
	\centering \epsfig{file=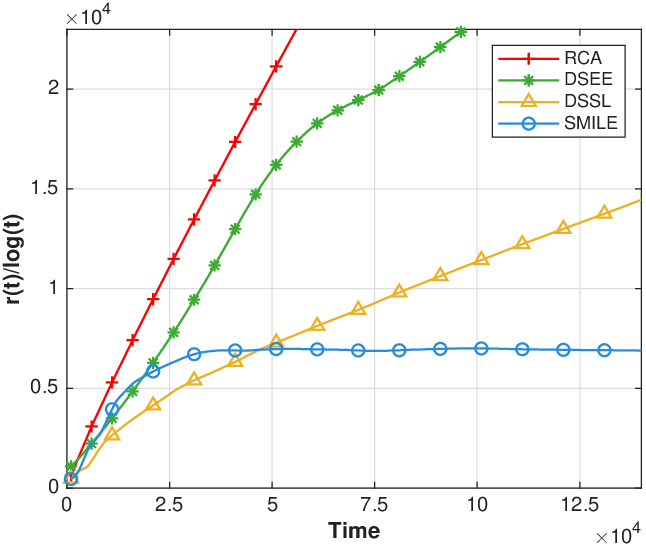,
    width=0.3\textwidth}
  \caption{The regret (normalized by $\log t$) as a function of time under RCA, DSEE, DSSL, and SMILE algorithms. 
  }
  \label{sim:sim2}
  \end{figure}

Finally, we compared SMILE with the dE3 algorithm \cite{nayyar2016regret} and the Game of Thrones (GoT) algorithm \cite{bistritz2021game} in a general setting where channels yield different expected rates for different cells. The dE3 algorithm requires inter-cell communication, as it performs a distributed auction in which cells observe each other's bids, while GoT explores all channels uniformly to enable agreement on an optimal allocation in a fully distributed manner. We set $L=S=4$ with a neighborhood graph in which cells 1 and 2 are neighbors, cells 1 and 3 are neighbors, and cell 4 has no neighbors. The instantaneous rates were generated as $r_{\ell,s}(t)=\mu_{\ell,s}+z_{\ell,s}(t)$, where $z_{\ell,s}(t)$ are i.i.d.\ Gaussian with zero mean and variance $\sigma^2=0.05$. For dE3 and GoT we used the parameter settings as used by their authors for i.i.d.\ channels (the setting for which they were designed). As shown in Fig.~\ref{sim:sim7}, SMILE substantially outperforms both algorithms, owing to its more efficient exploration structure.

\begin{figure}[htbp]
	\centering \epsfig{file=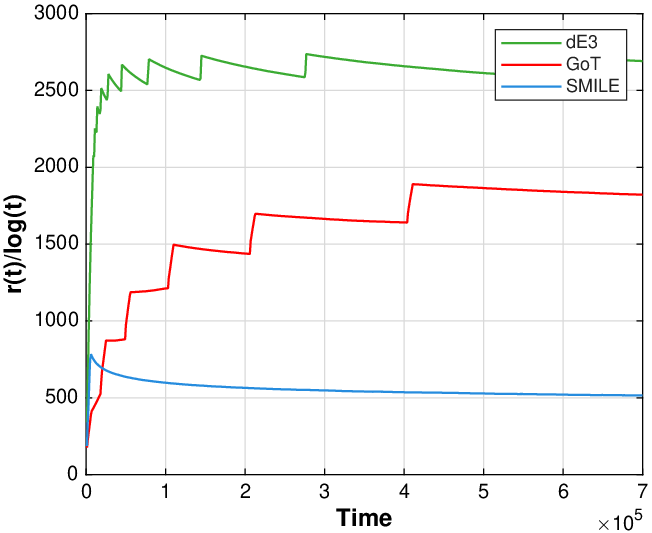,
    width=0.3\textwidth}
  \caption{The expected regret under dE3, GoT and SMILE versus time.}
  \label{sim:sim7}
  \end{figure}



\section{Conclusion}
We introduced a distributed learning framework for multi-cell spectrum access and sharing under spatially structured interference and restless Markovian channels. The proposed SMILE algorithm efficiently learns cell-specific channel rates while ensuring convergence to a stable allocation with provable logarithmic regret. Our approach generalizes previous models, accommodating both fully interfering and collision-free scenarios, and provides a practical, low-complexity solution for distributed spectrum management. Theoretical guarantees and numerical results together demonstrate the strong performance of the SMILE algorithm.

\section{Appendix}
In this appendix, we present the proof of Theorem 1.
\begin{definition}
    Let $T_1$ be the smallest integer for which, for every $t \geq T_1$, the following conditions hold:
    $$   
    E_{\ell,s} \leq \hat{E}_{\ell,s}(t) \quad \text{for all } \ell \in \mathcal{L},\, s \in \mathcal{S},
    $$
    and
    $$
    \hat{E}_{\ell,s}(t) \leq E_{\ell,s}^{(\max)} \quad \text{for all } \ell \in \mathcal{L},\, s \in \mathcal{A}_\ell.
    $$
    \end{definition}    

\begin{lemma}
\label{lemma 1}
    Assume that the SMILE algorithm is implemented as described in Section~\ref{sec: The Algorithm}. Then $E[T_1] < \infty$, and this bound does not depend on $t$.
\end{lemma}
\begin{proof}
    $E[T_1]$ can be expressed as
    {\small
    \begin{align}
    E[T_{1}]=&\sum \limits _{n=1}^{\infty } n \cdot Pr\left ({T_{1}= n }\right)= \sum \limits _{n=1}^{\infty }\Pr \left ({T_{1}\geq n }\right) 
    \notag \\
    =&\sum \limits _{n=1}^{\infty } \Pr \big (\bigcup \limits _{\ell \in \mathcal{L}} \bigcup \limits _{s\in \mathcal {A}_{\ell}} \bigcup \limits _{t=n}^{\infty }(\hat{E}_{\ell, s}(t)
    \notag \\
    < &E_{\ell, s} { \text {or }} \hat{E}_{\ell, s}(t)> E_{\ell, s}^{(\max)} { \text {or }} 
    \notag \\
    &\bigcup \limits _{\ell \in \mathcal{L}} \bigcup \limits _{s {\not \in }\mathcal {A}_{\ell}}\bigcup \limits _{t=n}^{\infty }(\hat{E}_{\ell, s}(t)< E_{\ell, s}) \big) 
    \notag \\
    \leq&\sum \limits _{\ell \in \mathcal{L}} \sum \limits _{s\in \mathcal {A}_{\ell}}\sum \limits _{n=1}^{\infty } \sum \limits _{t=n}^{\infty } \Pr \big (\hat{E}_{\ell, s}(t) 
    \notag \\
    &< E_{\ell, s} { \text {or }} \hat{E}_{\ell, s}(t)> E_{\ell, s}^{(\max)}\big) 
    \notag \\
    &+\sum \limits _{\ell \in \mathcal{L}} \sum \limits _{s {\not \in }\mathcal {A}_{\ell}}\sum \limits _{n=1}^{\infty } \sum \limits _{t=n}^{\infty } \Pr \big (\hat{E}_{\ell, s}(t)< E_{\ell, s}\big).
    \end{align}
    }
Note that it suffices to show that there exist constants $C>0$ and $\delta > 0$ such that, for all $\ell\in \mathcal {L}, s\in \mathcal {A}_{\ell}$ and for all $t\geq n$:
\begin{equation}
\label{lemma 1 proof: eq 1}
\Pr \big (\hat{E}_{\ell, s}(t)< E_{\ell, s} { \text { or }} \hat{E}_{\ell, s}(t)> E_{\ell, s}^{(\max)}\big)\leq C\cdot t^{-(2+\delta)},
\end{equation}
since then:
{\small
\begin{align}
&\hspace {-1.3pc}
\sum \limits _{\ell\in \mathcal {L}} \sum \limits _{s\in \mathcal {A}_{\ell}}\sum \limits _{n=1}^{\infty } \sum \limits _{t=n}^{\infty } \Pr \big (\hat {E}_{\ell, s}(t)< E_{\ell, s} { \text { or }} \hat {E}_{\ell, s}(t)> E_{\ell, s}^{(\max)}\big) 
\notag \\ 
\leq&L S C\left [{\sum \limits _{t=1}^{\infty } t^{-(2+ \delta)}+\sum \limits _{n=2}^{\infty }\sum \limits _{t=n}^{\infty } t^{-(2+ \delta)}}\right] 
\notag \\
\leq&LS C\left [{\sum \limits _{t=1}^{\infty } t^{-(2+ \delta)}+\sum \limits _{n=2}^{\infty }\int \limits _{n-1}^{\infty } t^{-(2+ \delta)}dl}\right] 
\notag \\
=&LS C\left [{\sum \limits _{t=1}^{\infty } t^{-(2+ \delta)}+\frac {1}{1+\delta }\sum \limits _{n=2}^{\infty }(n-1)^{-(1+\delta)}}\right]< \infty,
\end{align}
}
which is bounded independently of $t$. In the same manner, if we show that there exist constants $C>0$ and $\delta > 0$ such that, for all $\ell\in \mathcal {L}, s\notin \mathcal {A}_{\ell}$ and for all $t\geq n$: $\Pr \big (\hat {E}_{\ell, s}(t)< E_{\ell, s}\big)\leq C\cdot t^{-(2+\delta)}$, the statement is complete.

We begin to bound (\ref{lemma 1 proof: eq 1}). For cell $\ell$ and channel $s \in \mathcal{S}_\ell$, the event corresponding to the first inequality in (\ref{lemma 1 proof: eq 1}), namely $\hat{E}_{\ell,s}(t) < E_{\ell,s}$, implies:
{\small
\begin{align}
&\hspace {-1.5pc}
\max \bigg \{ \Delta _{\min }^{2}, \min \big \{ \min _{p \neq s} \{ (\hat{r}_{\ell,s}(t)-\hat{r}_{\ell,p}(t))^{2} \} \notag \\
&-\,\epsilon, \min _{q \in \mathcal{V}_{\ell, s}} \{ (\hat{r}_{\ell,s}(t)-\hat{r}_{q, s}(t))^{2} \} -\epsilon \big \} \bigg \} 
\notag \\
>&\min \big \{ \min _{p \neq s} \{ (\mu _{\ell, s}-\mu _{\ell, p})^{2} \}, \min _{q \in \mathcal{V}_{\ell, s}} \{ (\mu _{\ell, s}-\mu _{q,s})^{2} \} \big \},
\end{align}
}
which, after algebraic manipulations and using (\ref{delta}), entails that at least one of the following holds:
{\small
\begin{align}
\label{lemma 1 proof: eq 5}
\min _{p \neq s} \{ (\hat{r}_{\ell,s}(t) -\hat {r}_{\ell, p}(t))^{2} \} -\epsilon>&\min _{p \neq s} \{ (\mu _{\ell, s}-\mu _{\ell, p})^{2} \} 
\\
\min _{q \in \mathcal{V}_{\ell, s}} \{ (\hat{r}_{\ell,s}(t)-\hat{r}_{q, s}(t))^{2} \}  -\epsilon>&\min _{q \in \mathcal{V}_{\ell, s}} \{ (\mu _{\ell, s}-\mu _{q,s})^{2} \}.
\end{align}
}
Similarly, from the second inequality of (\ref{lemma 1 proof: eq 1}) we get that one of the following holds:
{\small
\begin{align} 
\label{lemma 1 proof: eq 6}
\min _{p \neq s} \{ (\hat{r}_{\ell,s}(t) - \hat {r}_{\ell, p}(t))^{2} \} -\epsilon<&\min _{p \neq s} \{ (\mu _{\ell, s}-\mu _{\ell, p})^{2} \}  -2\epsilon
\\
\min _{q \in \mathcal{V}_{\ell, s}} \{ (\hat{r}_{\ell,s}(t)-\hat{r}_{q, s}(t))^{2} \}  -\epsilon<&\min _{q \in \mathcal{V}_{\ell, s}} \{ (\mu _{\ell, s}-\mu _{q,s})^{2} \} -2\epsilon.
\end{align}
}
Let 
{\small
\begin{equation*}
\begin{array}{l}
s^{*} = \text {arg} \min  \limits_{p \neq s} \{(\mu _{\ell, s}-\mu _{\ell, p } )^{2}\}\;,\; 
\ell^{*} = \text {arg} \min \limits_{q \in \mathcal{V}_{\ell, s}} \{(\mu _{\ell, s}-\mu _{q,s} )^{2}\}, 
\\ (\hat{s})^{*} = \text {arg} \min \limits_{p \neq s} \{(\hat{r} _{\ell, s}(t)-\hat{r} _{\ell, p }(t) )^{2}\}, 
\\ 
(\hat{\ell})^{*} = \text {arg} \min \limits_{q \in \mathcal{V}_{\ell, s}}\{ (\hat{r} _{\ell, s}(t)-\hat{r} _{q,s}(t) )^{2}\}.
\end{array}
\end{equation*}
}
Note that we are not guaranteed that $s^{*}=(\hat{s})^{*}$ or that $\ell^{*}=(\hat{\ell})^{*}$, but from (\ref{lemma 1 proof: eq 5}) we get that one of the following holds:
{\small
\begin{align}
(\hat{r} _{\ell, s}(t)-\hat{r} _{\ell, s^* }(t) )^{2} - \epsilon &\geq (\hat{r} _{\ell, s}(t)-\hat{r} _{\ell, (\hat{s})^* } (t))^{2} - \epsilon \notag \\ &> (\mu _{\ell, s}-\mu _{\ell, s^* } )^{2}, \notag
\end{align}
}
\vspace{-0.5cm}
{\small
\begin{align}
(\hat{r} _{\ell, s}(t)-\hat{r} _{\ell^*, s }(t) )^{2} - \epsilon &\geq (\hat{r} _{\ell, s}(t)-\hat{r} _{(\hat{\ell})^*, s }(t) )^{2} - \epsilon  \notag \\ &> (\mu _{\ell, s}-\mu _{\ell^*, s } )^{2},
\end{align}
}
and from (\ref{lemma 1 proof: eq 6}) one of the following:
{\small
\begin{align}
    (\hat{r} _{\ell, s}(t)-\hat{r} _{\ell, (\hat{s})^* }(t) )^{2}  &< (\mu _{\ell, s}-\mu _{\ell, s^* } )^{2} -\epsilon \notag \\ &\leq (\mu _{\ell, s}-\mu _{\ell, (\hat{s})^* } )^{2} -\epsilon \notag
    \end{align}
    }
    \vspace{-0.5cm}
    {\small
\begin{align}
    (\hat{r} _{\ell, s}(t)-\hat{r} _{(\hat{\ell})^*, s } (t))^{2}  &< (\mu _{\ell, s}-\mu _{\ell^*, s } )^{2} -\epsilon \notag \\&\leq (\mu _{\ell, s}-\mu _{(\hat{\ell})^*, s} )^{2} -\epsilon.
\end{align}
}
Cascading the events written above we get:
{\small
\begin{align}
\label{lemma 1 proof: eq 2}
&\hspace {-1.5pc}
\Pr \big (\hat {E}_{\ell, s}(t)< E_{\ell, s} { \text {or }} \hat {E}_{\ell, s}(t)> D_{\ell, s}^{(\max)}~\big) 
\notag \\
\leq&\Pr \big ((\hat{r} _{\ell, s}(t)-\hat{r} _{\ell, s^* }(t) )^{2} - (\mu _{\ell, s}-\mu _{\ell, s^* } )^{2}>\epsilon \big) 
\notag \\
&+\, \Pr \big ((\hat{r} _{\ell, s}(t)-\hat{r} _{\ell^*, s }(t) )^{2} - (\mu _{\ell, s}-\mu _{\ell^*, s } )^{2} >\epsilon \big)
\notag \\
&+\, \Pr \big ( (\mu _{\ell, s}-\mu _{\ell, (\hat{s})^* } )^{2}-(\hat{r} _{\ell, s}(t)-\hat{r} _{\ell, (\hat{s})^* }(t) )^{2} >\epsilon \big)
\notag \\
&+\, \Pr \big ( (\mu _{\ell, s}-\mu _{(\hat{\ell})^*, s} )^{2}- (\hat{r} _{\ell, s}(t)-\hat{r} _{(\hat{\ell})^*, s } (t))^{2}>\epsilon \big).
\end{align}
}
Formally, each term represents the probability that the squared empirical difference and the squared true difference differ by more than $\epsilon$, in either direction, We look at the first term of (\ref{lemma 1 proof: eq 2}). Using conventional steps from set theory, it can be shown that:
{\small
\begin{align}
&\hspace {-1.6pc}
\Pr \big ((\hat{r} _{\ell, s}(t)-\hat{r} _{\ell, s^* }(t) )^{2} - (\mu _{\ell, s}-\mu _{\ell, s^* } )^{2}>\epsilon \big)  \notag\\ 
\leq&\big [\Pr \big (|(\hat {r}_{\ell, s}(t)- \hat {r}_{\ell, s^{*}}(t))
-(\mu _{\ell, s}- \mu _{\ell, s^{*}})|> \frac {\epsilon }{2(R+1)}\big) 
\notag \\
&+\,\Pr \big (|(\mu _{\ell, s}- \mu _{\ell, s^{*}})|>R+1\big)\big] \notag \\
&+ \big [\Pr \big (\mu _{\ell, s}>R' \big) \,+\,\Pr \big (\mu _{\ell, s^*}>R' \big) \notag \\ &+\Pr \big (|(\hat {r}_{\ell, s}(t)- \hat {r}_{\ell, s^{*}}(t))
-(\mu _{\ell, s}- \mu _{\ell, s^{*}})|>\frac {\epsilon }{2(R'+1)} \big)\big],
\end{align}
}
for every $R,R'>0$. We set $R=R'=r_{\max }$, which makes the second, third and fourth terms vanish. Consequently, the concentration bounds become:
{\small
\begin{align}
&\hspace {-1.6pc}
\Pr \big ((\hat{r} _{\ell, s}(t)-\hat{r} _{\ell, s^* }(t) )^{2} - (\mu _{\ell, s}-\mu _{\ell, s^* } )^{2}>\epsilon \big)
\notag \\
\label{lemma 1 proof: eq 3}
< &4 \cdot \max \bigg \{ \Pr \left({|\hat {r}_{\ell, s}(t)-\mu _{\ell, s}|> \frac {\epsilon }{4(r_{\max }+1)} }\right), 
\\&
\label{lemma 1 proof: eq 4}
\Pr \left({|\hat {r}_{\ell, s^{*}}(t)-\mu _{\ell, s^{*}}|> \frac {\epsilon }{4(r_{\max }+1)} }\right) \bigg \}.
\end{align}
}
Analogous bounds hold for all the terms in (\ref{lemma 1 proof: eq 2}). To bound (\ref{lemma 1 proof: eq 3}) and (\ref{lemma 1 proof: eq 4}) we use the results of Lezaud \cite{lezaud1998chernoff}:

\textit{Lemma 2:\cite{lezaud1998chernoff} Consider a finite-state, irreducible Markov chain $\{X_{t}\}_{t \geq 1}$ with state space $S$ matrix of transition probabilities $P$, an initial distribution $q$ and stationary distribution $\pi$. Let $N_{\textbf {q}}=\left \|{ \left({\frac {q^{(x)}}{\pi ^{(x)}}, x \in S}\right) }\right \|_{2}$. Let $\widehat {P}=P'P$ be the multiplicative symmetrization of $P$ where $P'$ is the adjoint of $P$ on $l_{2}(\pi)$. Let $\epsilon = 1-\lambda _{2}$, where $\lambda _{2}$ is the second largest eigenvalue of the matrix $P'$. $\epsilon$ will be referred to as the eigenvalue gap of $P'$. Let $f:S \rightarrow \mathcal {R}$ be such that $\sum \limits _{y \in S} \pi _{y}f(y)=0, \quad \|f\|_{2} \leq 1$ and  $0 \leq \|f\|_{2}^{2} \leq 1$ if $P'$ is irreducible. Then, for any positive integer $n$ and all $0< \lambda \leq 1$, we have:
$P \left ({\frac {\sum \limits _{t=1} ^{n}f(X_{t})}{n} \geq \lambda }\right) \leq N_{\textbf {q}}$.}

Let $\textbf {q}^{\ell, p}$ denote the initial distribution for channel $s$ and cell $\ell$. Then:
{\small
\begin{equation} 
N_{\textbf {q}}^{(\ell ,s)} = \left \|{ \left({\frac {q_{\ell ,s}^{r}}{\pi _{\ell ,s}^{r}}, r \in \mathcal{R}^{\ell ,s}}\right) }\right \|_{2} \leq \sum \limits _{r \in \mathcal{R}^{\ell ,s}} \left \|{ \frac {q_{\ell ,s}^{r}}{\pi _{\ell ,s}^{r}} }\right \|_{2} \leq \frac {1}{\pi _{min}}.
\end{equation}
}
Note that the empirical mean $\hat{r}_{\ell, s}(t)$ is based on $T_{\ell,s}^{\mathrm{EE}}(t)$ observations collected solely during the EE sub-epochs of the exploration phases. Consequently, the sample path underlying $\hat{r}_{\ell, s}(t)$ can be regarded as being generated by a Markov chain whose transition matrix matches that of the original channel $\{\ell, s \}$ so. This allows us to apply Lezaud’s result to bound (\ref{lemma 1 proof: eq 3}) and (\ref{lemma 1 proof: eq 4}). For equation (\ref{lemma 1 proof: eq 3}): 
We denote by $n^{\ell, s}_r(t)$ the number of times cell $\ell$ has observed state $r$ on channel $s$ up to time $t$. Then,
{\small
\begin{align}
&\hspace {-1.3pc}
\Pr \left({\hat {r}_{\ell ,s}(t) -\mu _{\ell, s} > \frac {\epsilon }{4(r_{\max }+1)} }\right) 
\notag \\
=&\Pr \left({\!\!\sum \limits _{r \in \mathcal {R}^{\ell, s}}\!\! r \cdot n_{r}^{\ell, s}(t)\!\!-\!\!T_{\ell,s}^{\mathrm{EE}}(t) \!\!\sum \limits _{r \in \mathcal {R}^{\ell, s}}\!\! r \cdot \pi _{\ell, s}^{r} \!\!>\!\!\! \frac {T_{\ell,s}^{\mathrm{EE}}(t)\cdot \epsilon }{4(r_{\max }+1)}}\right) 
\notag \\
\leq&\sum \limits _{r \in \mathcal {R}^{\ell, s}} \Pr \left({r \cdot n_{r}^{\ell, s}(t)\!\!-\!\!T_{\ell,s}^{\mathrm{EE}}(t) r \cdot \pi _{\ell, s}^{r} \!\!>\!\! \frac {T_{\ell,s}^{\mathrm{EE}}(t)\cdot \epsilon }{4(r_{\max }+1) |\mathcal {R}^{\ell, s}|} }\right) 
\notag \\
\leq&|\mathcal {R}^{\ell, s}| \cdot N_{\textbf {q}}^{(\ell, s)} \text {exp}\notag\\
&\left({\! \!-\!\!T_{\ell,s}^{\mathrm{EE}}(t) \!\cdot \!\frac {\epsilon ^{2}}{16(r_{\max }\!\!+\!1)^{2} \!\cdot \! r^{2} \!\cdot \! |\mathcal {R}^{\ell, s}|^{2} \!\cdot \! (\hat {\pi }_{\ell, s}^{r})^{2}} \!\!\cdot \! \frac {(1\!\!-\!\!\lambda _{\ell, s})}{12}\! }\right),
\end{align}
}
from (\ref{eq: exploration condition}) and (\ref{eq: exploration function}), we obtain: $T_{\ell,s}^{\mathrm{EE}}(t) > \frac {2}{I} \log (t)$ with $I$ defined in (\ref{I_kappa}). So we get,
{\small
\begin{align}
\label{proof eq 6}
\Pr \left({|\hat {r}_{\ell ,s}(t)-\mu _{\ell, s}|> \frac {\epsilon }{4(r_{\max}+1)} }\right) \leq \frac {|\mathcal{C}_{\max }|}{\pi _{\min }} \cdot t^{-2+\delta }. 
\end{align}
}
The same bound applies to (\ref{lemma 1 proof: eq 4}), and with the same steps, to all terms in (\ref{lemma 1 proof: eq 2}). The arguments for all $\ell\in \mathcal {L}, s {\not \in }\mathcal {A}_{\ell}$ is similar, and thus Lemma~\ref{lemma 1} follows.
\end{proof}

From Lemma~\ref{lemma 1}, we obtain that $T_1$ is finite.
Thus, by decomposing the expected regret in (\ref{eq: regret}) into the cases $t \leq T_1$ and $t > T_1$, we observe that the regret for $t \leq T_1$ remains constant and independent of $t$, i.e., $O(1)$.
Hence, we focus on deriving a bound for the expected regret when $t > T_1$.

Notice that, from the definition of $T_1$,
\begin{equation} 
\label{proof eq 1}
E_{\ell, s} \leq \hat {E}_{\ell, s}(t) \leq E_{\ell, s}^{(\max)}, 
\end{equation}
for all $\ell \in \mathcal{L}$, $s \in \mathcal {A}_{\ell}$ and LHS of the inequality holds for all $\ell \in \mathcal{L}$ and $s \in \mathcal{S}$. The lower bound guarantees that the exploration phases provide sufficient learning of the channel statistics, while the upper bound ensures that the channels are judiciously oversampled during these phases.

The regret for $t > T_1$ is given by:
{\small
\begin{align} 
\label{proof eq 2}
\mathcal {R}(t) \leq (t-T_{1}) \cdot \sum \limits _{\ell=1}^{L} \mu _{i,P(\ell)} - \mathbb {E} \left[{\sum \limits _{n=T_{1}+1}^{t} \sum \limits _{\ell=1}^{L} x_{\ell,\phi_{\ell}(n)}(n) }\right]
\end{align}
}
For ease of analysis, we bound (\ref{proof eq 2}) over $n = 1$ to $t$, where (\ref{proof eq 1}) (and the left-hand side for $s \notin \mathcal {A}_{\ell}$) holds for all $1 \leq n \leq t$. This construction provides an upper bound for (\ref{proof eq 2}):
{\small
\begin{align} 
\label{proof eq 3}
\mathcal {R}(t)\leq&(t-T_{1}) \cdot \sum \limits _{\ell=1}^{L} \mu _{\ell,P(\ell)} - \mathbb {E} \left[{\sum \limits _{n=T_{1}+1}^{t} \sum \limits _{\ell=1}^{L} x_{\ell,\phi_{\ell}(n)}(n)}\right] 
\notag \\
\leq&t \cdot \sum \limits _{\ell=1}^{L} \mu _{\ell,P(\ell)} - \mathbb {E} \left[{\sum \limits _{n=1}^{t} \sum \limits _{\ell=1}^{L} x_{\ell, \phi_{\ell}(n)}(n)}\right].
\end{align}
}
We denote by $T_{\ell, s}(t)$ the total number of transmissions of cell $\ell$ on channel $s$ up to time $t$ (where $x_{\ell, s}(n) = 0$ if cell $\ell$ did not attempt to access channel $s$ at time $n$). Hence, (\ref{proof eq 3}) takes the form
{\small
\begin{align} 
\label{proof eq 4}
\hspace {-1cm} 
\mathcal {R}(t)\leq&\sum \limits _{\ell=1}^{L} \sum \limits _{s=1}^{S} \left({\mu _{\ell, s} \cdot E[T_{\ell, s}(t)] - E\left[{\sum \limits _{n=1}^{t} x_{\ell, s}(n)}\right] }\right) \\ 
\label{proof eq 5}
&+\, \left({t \cdot \sum \limits _{\ell=1}^{L} \mu _{\ell,P(\ell)} - \sum \limits _{\ell=1}^{L} \sum \limits _{s=1}^{S} \mu _{\ell, s} \cdot E[T_{\ell, s}(t)] }\right),
\end{align}
}
We refer to equation~(\ref{proof eq 4}) as the regret due to the transient effect, i.e., when the initial state of the channel may not correspond to its stationary distribution. Similarly, equation~(\ref{proof eq 5}) represents the regret caused by not playing the stable allocation. Moreover, the sum of (\ref{proof eq 4}) and (\ref{proof eq 5}) can be decomposed into the sum of three distinct regret components, corresponding to the three phases described in Section~\ref{sec: The Algorithm}:
\begin{equation} 
\label{proof regrets sum}
\mathcal {R}(t) = \mathcal {R}^{ER}(t)+ \mathcal {R}^{A}(t) + \mathcal {R}^{EI}(t).
\end{equation}
where $\mathcal {R}^{ER}(t), \mathcal {R}^{A}(t),\mathcal {R}^{EI}(t)$ correspond to the regrets incurred during the exploration, allocation, and exploitation phases, respectively .Next, we derive separate bounds for the regret in each of these three phases.

\subsection{Regret in the Exploration Phases:}
To bound the two terms that compose the regret during the exploration phase, we first define $N_{\ell,s}^{ER}(t)$ as the number of exploration phases performed by cell $\ell \in \mathcal{L}$ on channel $s \in \mathcal{S}$ up to time $t$.
To bound $N_{\ell,s}^{ER}(t)$, note that according to Section~\ref{ssec: exploration phase}, the total number of samples from the exploration phases in subepochs EE for cell $\ell$ on channel $s$ up to time $t$ is given by:
{\small
\begin{equation} 
T_{\ell,s}^{\mathrm{EE}}(t) = \sum \limits _{n=1}^{N_{\ell,s}^{ER}(t)} 4^{n-1} = \frac {1}{3} (4^{N_{\ell,s}^{ER}(t)}-1).
\end{equation}
}
Since $t > T_1$ (so that (\ref{proof eq 1}) holds) and we are in an exploration phase ((\ref{eq: exploration condition}) holds), we obtain that
$T_{\ell,s}^{\mathrm{EE}}(t) < \mathcal{E}_{\ell, s} \cdot \log (t)$, where $\mathcal{E}_{\ell, s}$ is defined in (\ref{eq: {E}_l,s}).
Combining these results yields:
\begin{equation} 
\label{proof ssec exploration eq 1}
N_{\ell,s}^{ER}(t) \leq \lfloor \log _{4}(3\mathcal{E}_{\ell, s}\log (t)+1) \rfloor +1.
\end{equation}
Next, we use this bound together with the maximum hitting time of cell $\ell$ on channel $s$, to bound the total time spent by cell $\ell$ exploring channel $s$, including both RE and EE, denoted by $N_{\ell,s}^{(O)}(t)$:
{\small
\begin{align} 
\label{proof ssec total exploration time}
E[N_{\ell, s}^{(O)}(t)]\leq&\sum \limits _{n=0}^{N_{\ell,s}^{ER}(t)-1}(4^{n}+M^{\ell,s}_{\max }) 
\notag \\
=&\frac {1}{3}(4^{N_{\ell,s}^{ER}(t)}-1)+ M^{\ell,s}_{max} \cdot N_{\ell,s}^{ER}(t) 
\notag \\
\leq&\frac {1}{3} [4(3\mathcal{E}_{\ell,s}\cdot \log (t)+1)-1] 
\notag \\
&+\, M^{\ell,s}_{\max } \cdot \left( \lfloor \log _{4}(3\mathcal{E}_{\ell, s}\log (t)+1) \rfloor +1 \right).
\end{align}
}
We now bound the regret caused by channel switching.
For that, we use the following lemma:

\textit{Lemma 3:
\cite{anantharam2003asymptotically} Consider an irreducible, aperiodic Markov chain with state space $X$, a matrix of transition probabilities $P$ , an initial distribution $\overrightarrow {q}$ which is positive in all states, and stationary distribution $\overrightarrow {\pi } (\pi _{x}$ is the stationary probability of state $x$). The state (reward) at time $t$ is denoted by $x(t)$. Let $\mu$ denote the mean reward. If we play the chain for an arbitrary time $T$ , then there exists a value $A_{p}$, such that: $E\left[{\sum \limits _{t=1}^{T}x(t)-\mu T}\right] \leq A_{p}$.}

\vspace{0.1cm}
\noindent
And from \cite{anantharam2003asymptotically}, $A_{p} \leq (\min _{x \in S}\pi _{x})^{-}1 \sum \limits _{x \in S} x$.

Lemma~3 establishes an upper bound on the deviation of a Markov chain from its stationary distribution, a phenomenon we refer to as the transient effect.
Since each cell and channel pair has its own independent exploration phases, no channel switching occurs within a specific exploration phase (see Section~\ref{ssec: exploration phase}).
Therefore, by applying Lemma~3, we can bound the regret due to channel switching (\ref{proof eq 4}) during the exploration phase by:
{\small
\begin{equation} 
\label{proof ssec exploration eq 2}
Q_{\max } \cdot \left({\sum \limits _{\ell=1}^{L} \sum \limits _{s=1}^{S} (\lfloor \log _{4}(3\mathcal{E}_{\ell, s}\log (t)+1) \rfloor +1) }\right).
\end{equation}
}
Next, we bound the regret caused by sub-optimal allocation (\ref{proof eq 5}), during the exploration phase.
When cell $\ell$ explores channel $s$, it contributes to the regret in two ways:
1.  {cell $\ell$ does not transmit on its stable channel, increasing the regret by $\mu _{\ell,P(\ell)}-\mu _{\ell, s}$; and
2.  if $s$ is the stable channel of a neighbor $q$, then due to a collision, the regret further increases by $\mu _{q, s}$.
Combining these two effects, we can bound (\ref{proof eq 5}) in the exploration phases by:
{\small
\begin{equation}
\label{proof ssec exploration eq 3}
\sum \limits _{\ell=1}^{L} \sum \limits _{s=1}^{S} \bigg (E[N_{\ell,s}^{(O)}(t)] \cdot (\mu _{\ell,P(\ell)} + \! \! \! \! \sum_{q \in P^{-1}(s) \cap \mathcal{N}_\ell } \! \! \! \! \left[ \mu _{q,s} \right] - \mu _{\ell,s}) \bigg),
\end{equation}
}
where $P^{-1}(s)$ denotes the set of cells for which channel $s$ is their stable channel.
By substituting (\ref{proof ssec total exploration time}) we can upper-bound (\ref{proof ssec exploration eq 3}).
Finally, combining (\ref{proof ssec exploration eq 2}) and (\ref{proof ssec exploration eq 3}), we obtain a bound on the first term in (\ref{proof regrets sum}):
{\small
\begin{align} 
\mathcal {R}^{ER}(t)\leq&Q_{\max } \cdot \left({\sum \limits _{\ell=1}^{L} \sum \limits _{s=1}^{S} (\lfloor \log _{4}(3\mathcal{E}_{\ell,s}\log (t)+1) \rfloor +1) }\right) 
\notag \\
&+\, \sum \limits _{\ell=1}^{L} \sum \limits _{s=1}^{S} \bigg (\! (4\mathcal{E}_{\ell, s} \cdot \log (t) + 1 
    \notag \\ 
    &+\, M_{\max }^{\ell, s}\big (\lfloor \log _{4}(3\mathcal{E}_{\ell, s}\log (t)+1) \rfloor +1 \big) \bigg)  \cdot  
\notag \\ & \qquad \qquad \quad \; \bigg(\mu _{\ell,P(\ell)}+ \! \! \! \! \sum_{q \in P^{-1}(s) \cap \mathcal{N}_\ell } \! \! \! \! \left[ \mu _{q,s} \right] \!\! -\!\! \mu _{\ell,s} \! \bigg), 
\end{align}
}
which corresponds to the first and second terms on the right-hand side of (\ref{eq: regret bound}).

\subsection{Regret in the Allocation Phases:}
First, we bound the total time spent in allocation phases up to time $t$, denoted by $T_A(t)$.

Since an allocation phase takes place (if at all) only after an exploration phase, the number of allocation phases up to time $t$, denoted by $N^A(t)$, is bounded above by the total number of exploration phases (by all cells) up to time $t$; that is,
{\small
\begin{align} 
\label{proof ssec allocation eq 1}
N^{A}(t) &\leq\sum \limits _{\ell=1}^{L} \sum \limits _{s=1}^{S} N_{\ell,s}^{ER}(t) 
\sum \limits _{s=1}^{S} \lfloor \log _{4}(3\mathcal{E}_{\ell,s}\log (t)+1) \rfloor +1 ,
\end{align}
}
where the last inequality follows from using (\ref{proof ssec exploration eq 1}).
To bound the number of time indices that each allocation phase takes we formulate and prove the following lemma:

\textit{Lemma 4:}
The number of time indices required to complete each allocation phase, as formulated in Section~\ref{ssec: allocation phase}, is upper bounded by $2L \cdot S$
\begin{proof}
    Note that the number of iterations required to reach a stable allocation is $L\cdot S$, since each cell needs at most $S$ iterations to be matched to a channel. Moreover, since each iteration contains (at most) two sub-epochs, S1 and S2, it takes 2 time indices. Therefore, the total number of time indices that the allocation phase takes is at most $2L\cdot S$.
\end{proof}
Therefore, the duration of the allocation phase is finite with respect to $t$ and depends only on the number of cells and channels.
Combining (\ref{proof ssec allocation eq 1}) with Lemma 4, the total time spent in allocation phases up to time $t$ is bounded as follows:
{\small
\begin{align} 
E[T_{A}(t)]\leq& 2LS\left({\sum \limits _{\ell=1}^{L} \sum \limits _{s=1}^{S} \lfloor \log _{4}(3\mathcal{E}_{\ell,s}\log (t)+1) \rfloor +1 }\right).
\end{align}
}
We now proceed to bound the contribution to the regret from the allocation phase. In each allocation phase, the number of channels switching is at most $2L\cdot S$. Therefore, the regret due to the transient effect (\ref{proof eq 4}) is bounded by
{\small
\begin{align}
\label{proof ssec allocation eq 2}
Q_{\max } \cdot 2LS \cdot \left({\sum \limits _{\ell=1}^{L} \sum \limits _{s=1}^{S} (\lfloor \log _{4}(3\mathcal{E}_{\ell,s}\log (t)+1) \rfloor +1) }\right). 
\end{align}
}
While the regret resulting from sub-optimal allocation (\ref{proof eq 5}) can be bounded by:
{\small
\begin{align} 
\label{proof ssec allocation eq 3}
E[T_{A}(t)] &\cdot \left({\sum \limits _{\ell=1}^{L} \mu _{\ell,P(\ell)} }\right)\leq \bigg [2LS \cdot \bigg (\sum \limits _{\ell=1}^{L} \sum \limits _{s=1}^{S} 
\notag \\
&\lfloor \log _{4}(3\mathcal{E}_{\ell,s}\log (t)+1) \rfloor +1 \bigg) \bigg] \cdot \left({\sum \limits _{\ell=1}^{L} \mu _{\ell,P(\ell)} }\right).
\end{align}
}
Using (\ref{proof ssec allocation eq 2}), (\ref{proof ssec allocation eq 3}) we conclude:
{\small
\begin{align} 
\mathcal {R}^{A}(t)\leq&2Q_{\max } \cdot LS \cdot \bigg (\sum \limits _{\ell=1}^{L} \sum \limits _{s=1}^{S} 
(\lfloor \log _{4}(3\mathcal{E}_{\ell,s}\log (t) +1) \rfloor +1) \bigg)\notag \\
& 
\hspace{-1.3cm}+\, \bigg [2LS \cdot \bigg (\sum \limits _{\ell=1}^{L} \sum \limits _{s=1}^{S} 
\lfloor \log _{4}(3\mathcal{E}_{\ell,s}\log (t) +1) \rfloor+1 \bigg) \bigg] \cdot \left({\sum \limits _{\ell=1}^{L} \mu _{\ell,P(\ell)} }\right), 
\end{align}
}
which matches the third and fourth terms on the RHS of (\ref{eq: regret bound}).

\subsection{Regret in the Exploitation Phases:}
As described in Section~\ref{ssec: exploitation phase}, the duration of the $n^{th}$ exploitation phase is $2\cdot 4^{(n-1)}$. Thus, the total time spent during the exploitation phase is:
{\small
\begin{equation} 
\sum \limits _{n=1}^{N^{EI}(t)} 2 \cdot 4^{n-1} = \frac {2}{3} (4^{N^{EI}(t)}-1) \leq t.
\end{equation}
}
Thus, we can bound the number of exploitation phases up to time $t$ by
{\small
\begin{equation} 
\label{proof ssec explitaion eq 11}
N^{EI}(t) \leq \lceil \log _{4}\left({\frac {3}{2}t+1}\right) \rceil.
\end{equation}
}
First, we upper-bound the regret caused by the transient effect during the exploitation phases. In each exploitation phase, there are no channel switchings, as each cell exploits its allocated channel.
We denote by $N^{EI}{\ell, s}(t)$ the total number of exploitation phases in which cell $\ell$ was allocated to channel $s$ up to time t. Notice that $\sum_{s=1}^{S} N^{EI}_{\ell, s}(t) = N^{EI}(t)$. Therefore using Lemma~3, (\ref{proof eq 4}) is bounded by
{\small
\begin{align}
\label{proof ssec explitaion eq 10}
Q_{\max}\sum_{\ell=1}^L \sum_{s=1}^S N^{EI}_{\ell, s}(t) =& L\cdot Q_{\max } \cdot N^{EI}(t) \notag \\  \leq& L \cdot  Q_{\max } \cdot \lceil \log _{4}\left({\frac {3}{2}t+1}\right) \rceil.
\end{align}
}
As for the regret caused by not playing the stable allocation (which we refer to as a sub-optimal) during the exploitation phases, we denote by $t_n$ the starting time of the $n^{th}$ exploitation phase, by $T_s(t)$ the total time spent in sub-optimal exploitation phases, and by $P_s(n)$ the probability that the $n^{th}$ exploitation phase is sub-optimal. Using these notations, 
{\small
\begin{equation}
\label{proof ssec explitaion eq 1}
\begin{array}{l}
E[T_s(t)]=\sum \limits _{n=1}^{N^{EI}(t)} 2 \cdot 4^{n-1} \cdot P_{s}(n) 
\\
\leq
\sum \limits _{n=1}^{\lceil \log _{4}\left({\frac {3}{2}t+1}\right) \rceil } 2 \cdot 4^{n-1} \cdot P_{s}(n) 
\leq\sum \limits _{n=1}^{\lceil \log _{4}\left({\frac {3}{2}t+1}\right) \rceil } 2t_{n} \cdot P_{s}(n). 
\end{array}
\end{equation}
}
To bound the second term of the regret, and thereby complete the proof of Theorem~\ref{theorem 1}, it remains to show that:
\begin{equation} 
\label{proof ssec explitaion eq 7}
P_{s}(n) \leq B \cdot t_{n}^{-1}.
\end{equation}
A sub-optimal exploitation phase may occur if the preceding allocation phase results in an incorrect allocation. This, in turn, can happen due to one of the following reasons. First, cell $\ell$ may fail to correctly identify the order of its $D_{\ell} + 1$ best channels before entering the allocation phase; we denote this event by $Y_{\ell}$. Second, a cell with a higher expected rate on channel $s$ than its neighbor may not be correctly identified entering the allocation phase; this event is denoted by $Z_s$. These events are defined explicitly as follows:
{\small
\begin{align} 
Y_{\ell}(t_{n})=&\bigcup \limits _{s \in \mathcal {S}_{\ell}} \bigcup \limits _{p \in \mathcal {S}} \big \{\hat {r}_{\ell,s}(t_{n})< \hat {r}_{\ell,p}(t_{n}) | \mu _{\ell,s}>\mu _{\ell,p} \big \} 
 \\ 
Z_{s}(t_{n})=&\bigcup \limits _{{q \in \mathcal {L}}}\big \{\hat {r}_{\ell,s}(t_{n})< \hat {r}_{q,s}(t_{n})|\mu _{\ell,s} = \max _{p \in \mathcal {V}_{\ell, s}} \mu _{p,s} \;, \notag \\ & \; \, \qquad  \mu_{l,s} > \mu_{q,s} \big \},
\end{align}
}
Based on these events, the probability $P_s(n)$ is given by
\begin{equation} 
P_{s}(n) \triangleq \Pr \big (\bigcup \limits _{\ell \in \mathcal {L}} Y_{\ell}(t_{n}) { \text { or }} \bigcup \limits _{s \in \mathcal {S}} Z_{s}(t_{n}) \big).
\end{equation}
Using the union bound we have: 
{\small
\begin{align}
\label{proof ssec explitaion eq 2}
&\hspace {-1.6pc}\Pr \big (\bigcup \limits _{\ell \in \mathcal {L}} Y_{\ell}(t_{n}) { \text { or }} \bigcup \limits _{s \in \mathcal {S}} Z_{s}(t_{n}) \big) 
\notag \\
\leq&L S \cdot (\max_{\ell \in \mathcal{L}} D_{\ell}) \cdot \Pr \big (\hat {r}_{\ell,s}(t_{n})< \hat {r}_{\ell,p}(t_{n}) | \mu _{\ell,s}>\mu _{\ell,p} \big) 
\\
\label{proof ssec explitaion eq 3}
&+\, LS~\cdot \Pr \big (\hat {r}_{\ell,s}(t_{n})< \hat {r}_{q,s}(t_{n})|\mu _{\ell,s} > \mu _{q,s} \big). 
\end{align}
}
In order to bound (\ref{proof ssec explitaion eq 2}) and (\ref{proof ssec explitaion eq 3}), we first define $C_{t,v} = \sqrt{\kappa \log(t) / v}$. From (\ref{proof ssec explitaion eq 2}), it follows that at least one of the following conditions must hold:
{\small
\begin{align} 
\label{proof ssec explitaion eq 4}
\hat {r}_{\ell,s}(t_{n})\leq&\mu _{\ell,s}- C_{t_{n},T_{\ell,s}^{EE}(t)} 
\\ 
\label{proof ssec explitaion eq 5}
\hat {r}_{\ell,p}(t_{n})\geq&\mu _{\ell,p}+ C_{t_{n},T_{\ell,p}^{EE}(t)} 
\\ 
\label{proof ssec explitaion eq 6}
\mu _{\ell,s}< &\mu _{\ell,p}+ C_{t_{n},T_{\ell,p}^{EE}(t)}+ C_{t_{n},T_{\ell,s}^{EE}(t)}. 
\end{align}
}
Next, we show that the probability of event (\ref{proof ssec explitaion eq 6}) is zero.
{\small
\begin{align}
&\hspace {-1.3pc}\Pr \big (\mu _{\ell,s} < \mu _{\ell,p}+ C_{t_{n},T_{\ell,p}^{EE}(t_n)}+ C_{t_{n},T_{\ell,s}^{EE}(t_n)} \big) 
\notag \\
=&\Pr \left({\mu _{\ell,s} - \mu _{\ell,p} < \sqrt {\frac {\kappa \log t_{n}}{T_{\ell,p}^{EE}(t_n)}}+\sqrt {\frac {\kappa \log t_{n}}{T_{\ell,s}^{EE}(t_n)}} }\right) 
\notag \\
\leq&\Pr \left({\mu _{\ell,s} - \mu _{\ell,p} < 2\sqrt {\frac {\kappa \log t_{n}}{\min \left \{{T_{\ell,s}^{EE}(t_n), T_{\ell,p}^{EE}(t_n)}\right \}}} }\right) 
\notag \\
\leq&\Pr \left({\min \left \{{T_{\ell,s}^{EE}(t_n), T_{\ell,p}^{EE}(t_n)}\right \} < \frac {4\kappa}{(\mu _{\ell,s} - \mu _{\ell,p})^{2}}\log (t_{n}) }\right).
\end{align}
}
Combining (\ref{proof eq 1}) with the reverse of (\ref{eq: exploration condition}) (which holds since we started an allocation phase), also with the fact that $E_{\ell,s} \geq E_{\ell,s}^{(R)}$ and get:
{\small
\begin{align} 
T_{\ell,s}^{EE}(t_n)>&\frac {4\kappa}{\min \limits_{r \neq s} \{ (\mu _{\ell,s}-\mu _{\ell,r })^{2}\}}\log (t_{n}) 
\notag \\
\geq&\frac {4\kappa}{(\mu _{\ell,s} - \mu _{\ell,p})^{2}} \log (t_{n})
\end{align}
}
If $s \in S_{\ell}$:\vspace{-0.8em}
{\small
\begin{align}
T_{\ell,p}^{EE}(t_n)>&\frac {4\kappa}{\min \limits_{r \neq p } \{ (\mu _{\ell,p}-\mu _{\ell,r})^{2} \}}\log (t_{n}) \notag \\ \geq&\frac {4\kappa}{(\mu _{\ell,s} - \mu _{\ell,p})^{2}} \log (t_{n}),
\end{align}
}
otherwise,\vspace{-0.8em}
{\small
\begin{align}
T_{\ell,p}^{EE}(t_n)>&\frac {4\kappa}{(\mu_{\ell,p}-\min\limits_{r \in S_{\ell}} \mu _{\ell, r})^2}\log (t_{n}) \notag \\ \geq&\frac {4\kappa}{(\mu _{\ell,s} - \mu _{\ell,p})^{2}} \log (t_{n}).
\end{align}
}
Hence, the probability of event (\ref{proof ssec explitaion eq 6}) is zero. 

Using Lezaud’s result (Lemma 3), we now bound (\ref{proof ssec explitaion eq 4}) and (\ref{proof ssec explitaion eq 5}). By following steps similar to those used above to bound (\ref{lemma 1 proof: eq 3}), and using (\ref{I_kappa}), we obtain:
{\small
\begin{align} 
\Pr \big (\hat {r}_{\ell,s}(t_{n})\leq&\mu _{\ell,s}- C_{t_{n},v_{\ell,s}} \big) 
\notag \\
\leq&\frac {|\mathcal {R}^{\ell,s}|}{\pi _{\min }} t^{- \frac {\kappa \bar {\lambda }_{\min }}{28 \mathcal{C}_{\max }^{2} \overline{R}_{\max }^{2} \hat {\pi }_{\max }^{2} } } 
 = \frac {|\mathcal {R}^{\ell,s}|}{\pi _{\min }} \cdot t^{-1}
 \\
\Pr \big (\hat {r}_{\ell,p}(t_{n})\geq&\mu _{\ell,p}+ C_{t_{n},v_{\ell,p}} \big) 
\notag \\
\leq&\frac {|\mathcal {R}^{\ell,p}|}{\pi _{\min }} t^{- \frac {\kappa \bar {\lambda }_{\min }}{28 \mathcal{C}_{\max }^{2} \overline{R}_{\max }^{2} \hat {\pi }_{\max }^{2} }}
 = \frac {|\mathcal {R}^{\ell,s}|}{\pi _{\min }} \cdot t^{-1}
. \end{align}
}
Therefore (\ref{proof ssec explitaion eq 2}) is bounded by:
{\small
\begin{align}
L S \cdot (\max_{\ell \in \mathcal{L}} D_{\ell}) \cdot \frac {2\mathcal{C}_{\max }}{\pi _{\min }} \cdot t^{-1}.
\end{align}
}
By applying similar arguments, (\ref{proof ssec explitaion eq 3}) can be bounded, this time leveraging the fact that $E_{\ell,s} \geq E_{\ell,s}^{(C)}$. Consequently, Equation (\ref{proof ssec explitaion eq 7}) can be bounded:
{\small
\begin{align} 
\label{proof ssec explitaion eq 8}
\Pr \big (\bigcup \limits _{\ell \in \mathcal {L}} Y_{\ell}(t_{n}) & { \text { or }} \bigcup \limits _{s \in \mathcal {S}} Z_{s}(t_{n}) \big) \!\! \notag\\ &\leq (L S \cdot (\max_{\ell \in \mathcal{L}} D_{\ell})\!+\!LS) \frac {2\mathcal{C}_{\max }}{\pi _{\min }} \cdot t^{-1} 
\end{align}
}
Using (\ref{proof ssec explitaion eq 8}), we can bound (\ref{proof ssec explitaion eq 1}). Hence, the regret due to sub-optimal exploitation phases is bounded by:
{\small
\begin{equation} 
\label{proof ssec explitaion eq 9}
2 \left({\sum \limits _{\ell=1}^{M} \mu _{\ell,P(\ell)}}\right) (L S \cdot (\max_{\ell \in \mathcal{L}} D_{\ell})+LS) \frac {2\mathcal{C}_{\max }}{\pi _{\min }} \cdot \lceil \log _{4}\left({\frac {3}{2}t+1}\right) \rceil.
\end{equation}
}
By combining the bounds of the two regret components, (\ref{proof ssec explitaion eq 9}) and (\ref{proof ssec explitaion eq 10}), the total regret caused by the exploitation phases is given by:
{\small
\begin{align} 
\mathcal {R}^{EI}(t)\leq& L \cdot Q_{\max } \cdot \lceil \log _{4}\left({\frac {3}{2}t+1}\right) \rceil + 2 \left({\sum \limits _{\ell=1}^{L} \mu _{\ell,P(\ell)}}\right) 
\notag \\
&(L S \cdot (\max_{\ell \in \mathcal{L}} D_{\ell})+LS) \frac {2\mathcal{C}_{\max }}{\pi _{\min }} \cdot \lceil \log _{4}\left({\frac {3}{2}t+1}\right) \rceil, 
\end{align}
}
which coincides with the two last terms on the RHS of (\ref{eq: regret bound}).

\bibliographystyle{ieeetr}
\bibliography{Di_Co_SPP_bib}

\end{document}